\documentclass{article}

    \usepackage[final]{neurips_2022}

\usepackage[utf8]{inputenc} 
\usepackage[T1]{fontenc}    
\usepackage{hyperref}       
\usepackage{url}            
\usepackage{booktabs}       
\usepackage{amsfonts}       
\usepackage{nicefrac}       
\usepackage{microtype}      
\usepackage{xcolor}         

\usepackage{graphicx}
\usepackage{amsthm}
\usepackage{amssymb}
\usepackage{amsmath}
\usepackage{mathtools}
\usepackage{dirtytalk}
\usepackage{subcaption}
\usepackage{wrapfig}
\usepackage{placeins}
\usepackage{color}
\usepackage{soul}

\usepackage[noend,ruled,vlined,linesnumbered]{algorithm2e}
\usepackage{setspace}

\SetCommentSty{mycommfont}

\theoremstyle{definition}
\newtheorem{definition}{Definition}
\newtheorem{assumption}{Assumption}
\newtheorem{lemma}{Lemma}
\newtheorem{proposition}{Proposition}
\newtheorem{theorem}{Theorem}

\newcommand{\repobase}{https://github.com/ido90/CeSoR}
\newcommand{\repo}[1]{\href{\repobase}{\underline{#1}}}
\newcommand{\code}[2]{\href{\repobase/blob/main/#1}{\underline{#2}}}
\newcommand{\pypi}[1]{\href{https://pypi.org/project/cross-entropy-method/}{\underline{#1}}}

\title{Efficient Risk-Averse Reinforcement Learning}

\author{
  Ido Greenberg \\ 
  Technion \\
  \texttt{gido@campus.technion.ac.il}
  \And
  Yinlam Chow \\
  Google Research \\
  \texttt{yinlamchow@google.com}
  \And
  Mohammad Ghavamzadeh \\
  Google Research \\
  \texttt{ghavamza@google.com}
  \And
  Shie Mannor \\
  Technion, Nvidia Research \\
  \texttt{shie@ee.technion.ac.il}
}

\begin{document}

\maketitle

\begin{abstract}
In risk-averse reinforcement learning (RL), the goal is to optimize some risk measure of the returns. A risk measure often focuses on the worst returns out of the agent's experience. As a result, standard methods for risk-averse RL often ignore high-return strategies. 
We prove that under certain conditions this inevitably leads to a local-optimum barrier, and propose a mechanism we call soft risk to bypass it. We also devise a novel cross entropy module for sampling, which (1) preserves risk aversion despite the soft risk; (2) independently improves sample efficiency. By separating the risk aversion of the sampler and the optimizer, we can \textit{sample} episodes with poor conditions, yet \textit{optimize} with respect to successful strategies.
We combine these two concepts in CeSoR -- Cross-entropy Soft-Risk optimization algorithm -- which can be applied on top of \textit{any} risk-averse policy gradient (PG) method. We demonstrate improved risk aversion in maze navigation, autonomous driving, and resource allocation benchmarks, including in scenarios where standard risk-averse PG completely fails.
Our results and CeSoR implementation are available on \repo{Github}.
The stand-alone cross entropy module is available on \pypi{PyPI}.

\end{abstract}


\section{Introduction}
\label{sec:intro}

Risk-averse reinforcement learning (RL) is important for high-stake applications,
such as driving, robotic surgery, and finance~\citep{risk_averse_rl_finance}. In contrast to risk-neutral RL, it optimizes a risk measure of the return random variable, rather than its expectation. A popular risk measure is the Conditional Value at Risk (CVaR), defined as $\text{CVaR}_\alpha(R) = \mathbb{E}\left[ R \,|\, R\le q_\alpha(R) \right]$, where $q_\alpha(R)=\inf\{x\,|\,F_R(x)\ge\alpha\}$ is the $\alpha$-quantile of the random variable $R$ and $F_R$ is its CDF.
Intuitively, CVaR measures the expected return below a specific quantile $\alpha$, also termed the risk level.
CVaR optimization is widely researched in the RL community, e.g., using adjusted policy gradient approaches (CVaR-PG)~\citep{cvar_via_sampling,cvar_options}.
In addition, CVaR is a coherent risk measure, and its optimization is equivalent to a robust optimization problem~\citep{cvar_vs_robustness}.

\begin{figure}[!ht]
\vspace{-10pt}
\centering
\begin{subfigure}{.32\textwidth}
  \centering
  \includegraphics[width=1.\linewidth]{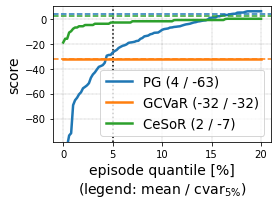}
  \caption{Guarded Maze}
  \label{fig:maze_test_scores}
\end{subfigure}
\begin{subfigure}{.32\textwidth}
  \centering
  \includegraphics[width=1.\linewidth]{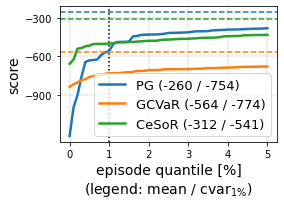}
  \caption{Driving Game}
  \label{fig:driving_test_scores}
\end{subfigure}
\begin{subfigure}{.31\textwidth}
  \centering
  \includegraphics[width=1.\linewidth]{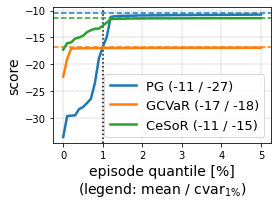}
  \caption{Servers Allocation}
  \label{fig:servers_test_scores}
\end{subfigure} \\
\begin{subfigure}{.31\textwidth}
  \centering
  \includegraphics[width=1.\linewidth]{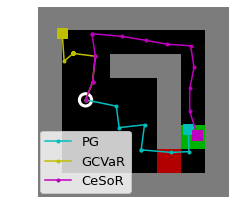}
  \caption{CeSoR learns to avoid the risk (red) and take the long path to the target (green), whereas GCVaR fails due to \textit{blindness to success}.}
  \label{fig:maze_nice}
\end{subfigure}\hspace{0.01\textwidth}
\begin{subfigure}{.315\textwidth}
  \centering
  \includegraphics[width=1.\linewidth]{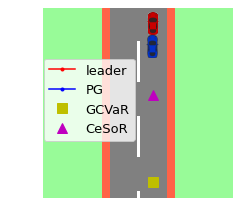}
  \caption{CeSoR maintains a safe margin from the leader, while PG has an accident and GCVaR maintains a too conservative distance.}
  \label{fig:driving_nice}
\end{subfigure}\hspace{0.01\textwidth}
\begin{subfigure}{.31\textwidth}
  \centering
  \includegraphics[width=1.\linewidth]{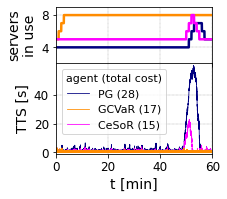}
  \caption{CeSoR handles the exceptional peak in user-requests without paying for as many servers as GCVaR, while PG fails to handle the peak.}
  \label{fig:servers_nice}
\end{subfigure}
\caption{\small Over 3 benchmarks, test results of 3 agents: risk-neutral PG, standard CVaR-PG (GCVaR,~\citet{cvar_via_sampling}), and our CeSoR. Top: the lower quantiles of the returns distributions. Bottom: sample episodes.}
\vspace{-15pt}
\label{fig:awesome}
\end{figure}

Since risk-averse RL aims to avoid the hazardous parts of the environment (e.g., dangerous areas in navigation), CVaR-PG algorithms typically sample a batch of $N$ trajectories (episodes), and then optimize w.r.t.~the mean of the $\alpha N$ trajectories with worst returns~\citep{cvar_via_sampling,cvar_trpo}. This approach suffers from two major drawbacks: (i) $1-\alpha$ of the batch is wasted and excluded from the optimization (where often $0.01\le \alpha\le0.05$), leading to sample inefficiency; (ii) focusing on the worst episodes inherently overlooks good agent strategies corresponding to high returns -- a phenomenon we refer to as the \textit{blindness to success}.

\textbf{An illustrative example -- the Guarded Maze}: Consider the Guarded Maze benchmark (visualized in Figure~\ref{fig:maze_nice}).
The goal is to reach the target zone (a constant location marked in green), resulting in a reward of $16$ points. However, the guarded zone (in red) \textit{may} be watched by a guard who demands a payment from any agent that passes by.
Every episode, the probability that a guard is present is $\phi_1=20\%$, and the payment is exponentially-distributed with average $\phi_2=32$. That is, the cost of crossing the guarded zone in a certain episode is $C = C_1\cdot C_2$, where $C_1\sim Ber(\phi_1), C_2\sim Exp(\phi_2)$ are independent and unknown to the agent.
The agent starts at a random point at the lower half, and every time-step, observing its location, it selects one an action: left, right, up or down, with an additive control noise. 
One point is deducted per step, up to $32$ deductions.

In this maze, the shortest path maximizes the average return; yet, the longer path is CVaR-optimal, since sometimes short cuts make long delays~\citep{tolkien}.
However, the standard CVaR-PG optimizer (GCVaR in Figure~\ref{fig:awesome}) suffers from blindness to success: in a batch of $N$ random episodes, the worst $\alpha N$ returns (e.g., for $\alpha=5\%$) usually correspond to either encountering a guard in the short path, or not reaching the goal at all. Hence, the desired long path is never even observed by the CVaR-PG optimizer, and cannot be learned.

Our key insight is that the variation in returns comes from both environment conditions (\emph{epistemic} uncertainty) and agent actions (\emph{aleatoric} uncertainty). We wish to focus on the \textit{low quantiles w.r.t.~the conditions} (e.g., a costly guard in the short path of the maze), yet to be exposed to the \textit{high quantiles w.r.t.~the strategies} (e.g., taking the long path in the maze). To that end, we devise two mechanisms: first, we use a soft risk-level scheduling method, which begins the training with risk neutrality $\alpha^\prime=1$, and gradually shifts the risk aversion to $\alpha^\prime=\alpha$. Second, we present a novel dynamic-target version of the Cross Entropy method (CE or CEM)~\citep{CE_tutorial}, aiming to sample the worst parts of the environment. That is, the CEM samples trajectories with more challenging or riskier conditions, and the soft risk feeds a larger part of them ($\alpha^\prime\ge\alpha$) to the CVaR-PG optimizer. Together, these constitute the Cross-entropy method for Soft-Risk optimization (\textbf{\textit{CeSoR}}). CeSoR can be applied on top of any CVaR-PG method to learn any differentiable model (e.g., a neural network).

To apply the CEM, we assume to have certain control over the environment conditions. For example, in driving we may choose the roads for collecting training data, or in any simulation we may control the environment parameters (e.g., $\phi_1,\phi_2$ in the Guarded Maze).
Note that (i) only the CE sampler (not the agent) is aware of the conditions; (ii) their underlying effect is unknown to the sampler and may vary with the agent throughout the training, hence the CEM needs to learn it adaptively.

\textbf{Contribution:} We present the following contribution for PG algorithms under risk-sensitive MDP problems (as defined in Section \ref{sec:preliminaries}):
\begin{enumerate}
    \item We analyze the phenomenon of {\em blindness to success} in the standard CVaR-PG, and show that it leads to a local-optimum barrier in certain environments (Section \ref{sec:analysis_blindness}).
    \item We analyze the potential increase in sample efficiency -- if we could sample directly from the tail of the returns distribution (Section \ref{sec:variance_reduction}).
    \item We introduce the CeSoR algorithm (Section \ref{sec:method}), which modifies any CVaR-PG method with: (i) a soft risk mechanism preventing blindness to success; (ii) a novel dynamic-CE method that over-samples the riskier realizations of the environment, increasing sample efficiency.
    \item We demonstrate the effectiveness of CeSoR in 3 risk-sensitive domains (Section~\ref{sec:experiments}), where it learns faster and achieves higher returns (both CVaR and mean) than the baseline CVaR-PG.
\end{enumerate}

\subsection{Related Work}
\label{sec:related_work}

Optimizing risk in RL is crucial to enforce safety in decision-making~\citep{safe_RL,rwrl_challenges}.
It has been long studied through various risk criteria, e.g., mean-variance~\citep{variance_TD,Prashanth13AC,Prashanth16VC,Xie18BC}, entropic risk measure~\citep{Borkar2002a,Borkar2014,exponential_bellman} and distortion risk measures~\citep{PG_for_distortion_risk}. \citet{pg_for_risk_measures} derived a PG method for general coherent risk measures, given their risk-envelope representation.

The CVaR risk measure specifically was studied using
value iteration \citep{cvar_vs_robustness} and distributional RL \citep{distributional_rl_risk,WorstCasePG,qt_opt} (also discussed in Appendix~\ref{sec:dRL}).
CVaR optimization was also shown equivalent to mean optimization under robustness \citep{cvar_vs_robustness}, motivating robust-RL methods \citep{robust_adversarial,cvar_adversarial}.
Yet, PG remains the most popular approach for CVaR optimization in RL \citep{cvar_via_sampling,cvar_trpo,cvar_options,risk_pg_convergence}, and can be flexibly applied to a variety of use-cases, e.g., mixed mean-CVaR criteria \citep{mean_cvar_pg} and multi-agent problems \citep{RMIX}.

Optimizing the CVaR for risk levels $\alpha\ll1$ poses a significant sample efficiency challenge, as only a small portion of the agent's experience is used to optimize its policy~\citep{adaptive_sampling_risk_aversion}. \citet{being_optimistic} used an exploration-based approach to address the sample efficiency.
Pessimistic sampling for improved sample efficiency was suggested heuristically by \citet{cvar_via_sampling} using a dedicated value function, but no systematic method was suggested to direct the pessimism level.
In this work, we use the CEM to control the sampled episodes around the desired risk level $\alpha$, and demonstrate CVaR optimization for as extreme levels as $\alpha=1\%$.
Note that unlike other CE-optimizers in RL~\citep{ce_for_policy_search,cem_gd}, we use the CEM for \textit{sampling}, to support a gradient-based optimizer.



\section{Problem Formulation}
\label{sec:preliminaries}

Consider a Markov Decision Process (MDP) $(S,A,P,\gamma,P_0)$, corresponding to states, actions, state-transition and reward distribution, discount factor, and initial state distribution, respectively.
For any policy parameter $\theta\in\mathbb R^n$, we denote by $\pi_\theta$ the parameterized policy that maps a state to a probability distribution over actions. Given a state-action-reward trajectory $\tau=\{(s_t,a_t,r_t)\}_{t=0}^T$, the trajectory total return is denoted by $R(\tau) = \sum_{t=0}^T \gamma^t r_t$. 
The expected return of a policy $\pi_\theta$ is defined as 
\begin{equation}
\label{eq:risk-neutral-obj}
J(\pi_\theta) = \mathbb{E}_{\tau\sim P^{\pi_\theta}}\left[ R(\tau) \right],  
\end{equation}
%
where $P^{\pi_\theta}(\tau)= P_0(s_0)\prod_{t=0}^{T-1} P(s_{t+1},r_t|s_t,a_t)\pi_\theta(a_t|s_t)$ is the probability distribution of $\tau$ induced by $\pi_\theta$. Under the risk-neutral objective, the PG method uses the gradient $\nabla_\theta J(\pi_\theta)$ to learn $\theta$, aiming to increase the probability of actions that lead to higher returns.
In contrast, CVaR-PG methods aim to optimize the risk-averse CVaR$_\alpha$ objective (w.r.t.~a given risk level $\alpha$):
%
\begin{equation}\label{eq:cvar}J_\alpha(\pi_\theta) = \mathbb{E}_{\tau\sim P^{\pi_\theta}}\left[ R(\tau) \,|\, R(\tau)\le q_\alpha(R|\pi_\theta) \right],
\end{equation}
where $q_\alpha(R|\pi_\theta)$ is the $\alpha$-quantile of the return random variable of policy $\pi_\theta$. Thus, CVaR-PG algorithms aim to improve the actions specifically for episodes whose returns are lower than $q_\alpha(R|\pi_\theta)$. 
Specifically, given a batch of $N$ trajectories $\{\tau_i\}_{i=1}^N$ whose empirical return quantile is $\hat{q}_\alpha=\hat{q}_\alpha(\{R(\tau_i)\}_{i=1}^N)$, the CVaR gradient estimation is given by~\citep{cvar_via_sampling}:
\begin{align}
\label{eq:GCVaR_grad}
    \nabla_\theta \hat{J}_\alpha(\{\tau_i\}_{i=1}^N;\, \pi_\theta) = \frac{1}{\alpha N} \sum_{i=1}^N w_i \cdot \pmb{1}_{R(\tau_i)\le \hat{q}_\alpha} \left(R(\tau_i)-\hat{q}_\alpha\right) \sum_{t=0}^T\nabla_\theta\log \pi_\theta(a_{i,t};s_{i,t}),
\end{align}
where $w_i = P^{\pi_\theta}(\tau_i) / f(\tau_i\,|\,\pi_\theta)$ is the importance sampling (IS) correction factor for $\tau_i$, if $\tau_i$ is sampled from a distribution $f \ne P^{\pi_\theta}$.
Specifically, as discussed below, we modify the sample distribution using the cross entropy method over a context-MDP formulation of the environment.





\textbf{Context-MDP:}
As mentioned above, we aim to focus on high-risk environment conditions.
To discuss the notion of conditions, given a standard MDP, we extend its formulation to a Context-MDP (C-MDP)~\citep{CMDP}, where the \textit{context} is a set of variables that capture (part or all of) the randomness of the original MDP.
We define the extension as $(S,A,\mathcal{C},P_C,\gamma, P_0, D_{\phi_0})$, where $C\in \mathcal{C}$ is sampled from the context space $\mathcal{C}$ according to the distribution $D_{\phi_0}$ parameterized by $\phi_0$, and $P_C(\cdot)=P(\cdot|C)$ is the transition and reward distribution conditioned on $C$. 
In a C-MDP, a context-trajectory pair is sampled from the distribution $P^{\pi_\theta}_{\phi_0}(C,\tau) = D_{\phi_0}(C)P_C^{\pi_\theta}(\tau)$, where $P_C^{\pi_\theta}(\tau) = P_0(s_0)\prod_{t=0}^{T-1} P_C(s_{t+1},r_t|s_t,a_t)\pi_\theta(a_t|s_t)$. The mean and $\text{CVaR}_\alpha$ objectives $J(\pi_\theta)$, $J_\alpha(\pi_\theta)$ in Equations~\eqref{eq:risk-neutral-obj}~and~\eqref{eq:cvar} are naturally generalized to C-MDP using the distribution $P^{\pi_\theta}_{\phi_0}(C,\tau)$.

Once we extend an MDP into a C-MDP, we can learn how to modify the context-distribution parameter $\phi$ to sample high-risk contexts and trajectories, focusing the training on high-risk parts of the environment and thus improving sample efficiency.
For this, we assume that certain aspects of the training environment (represented by $C$) can be controlled.
This assumption indeed holds in many practical applications -- in both simulated and physical environments.
For example, consider a data collection procedure for a self-driving agent training, which by default samples all driving hours uniformly: $C\sim U([0,24))$.
As the hour may affect traffic and driving patterns, a risk-averse driver would prefer to sample more experience in high-risk hours.
To that end, we could re-parameterize the uniform distribution as, say, $Beta(\phi)$ with ${\phi_0}=(1,1)$ (note that $Beta(1,1)$ is the uniform distribution), learn the high-risk hours, and modify $\phi$ to over-sample them.
As another example, in the Guarded Maze described above, we can control the parameters $\phi_1,\phi_2$ of the simulation.

\section{Limitations of CVaR-PG}
\label{sec:analysis}

Consider the standard CVaR-PG algorithm, which relies on Equation~\eqref{eq:GCVaR_grad} to apply PG for maximization of $J_\alpha(\pi_\theta)$ of~\eqref{eq:cvar}.
In this section, we analyze two major limitations of this algorithm.
Section~\ref{sec:analysis_blindness} analyzes the \emph{blindness to success} phenomenon, which may bring CVaR-PG learning to a local-optimum deadlock. This will motivate the soft-risk scheduling in Section~\ref{sec:method}.
Section~\ref{sec:variance_reduction} analyzes the potential increase in sample efficiency when the environmental context is sampled in correspondence to the tail of the returns distribution. This will motivate the cross-entropy sampler in Section~\ref{sec:method}.


While the analysis focuses on CVaR-PG methods, Appendix~\ref{sec:dRL} discusses Distributional RL algorithms for CVaR optimization, and demonstrates that similar limitations apply to these methods as well.


\subsection{Blindness to Success}
\label{sec:analysis_blindness}

We formally analyze how the \emph{blindness to success} phenomenon can bring the policy learning to a local-optimum deadlock by ignoring successful agent strategies.

Recall the $\alpha$-quantile of a return distribution $q_\alpha^{\pi}=\min\{r\,|\,F_{R(a)|\pi}(r)\ge\alpha\}$. We first introduce the notion of a \emph{tail barrier}, corresponding to a returns-distribution tail with a constant value.
%
\begin{definition}[Tail barrier]
\label{def:const_tail}
Let $\alpha\in(0,1]$. A policy $\pi$ has an $\alpha$-tail barrier if $\forall \alpha^\prime\in[0,\alpha]: q_{\alpha^\prime}^\pi = q_\alpha^\pi$.
\end{definition}
Note that in any environment with a discrete rewards distribution, a policy is prone to having a tail barrier for some $\alpha > 0$.
In existing CVaR-PG analysis~\citep{cvar_via_sampling}, such barriers are often overlooked by assuming continuous rewards.
For the Guarded Maze, Figure~\ref{fig:maze_scores_30} in the appendix demonstrates how a standard CVaR-PG exhibits a $0.9$-tail barrier, since as many as 90\% of the trajectories reach neither the target nor the guard, and thus have identical low returns.


A tail barrier has a destructive effect on CVaR-PG.
Consider a CVaR$_\alpha$ objective, and a policy $\pi$ with a $\beta$-tail barrier where $\beta>\alpha$.
Intuitively, any infinitesimal change of $\pi$ cannot affect the CVaR return, since the returns infinitesimally-above $q_\alpha^\pi$ are identical to those below $q_\alpha^\pi$.
That is, any tail barrier wider than $\alpha$ brings the CVaR-PG to a stationary point of type plateau.
More formally, consider $\nabla_\theta \hat{J}_\alpha$ of Equation~\eqref{eq:GCVaR_grad} with a $\beta$-tail barrier $\beta>\alpha$: any trajectory has either $\pmb{1}_{R(\tau_i)\le q_\alpha^\pi}=0$ (if its return is above $q_\alpha^\pi$) or $R(\tau_i)-q_\alpha^\pi=0$ (otherwise), hence the whole gradient vanishes.
Such a loss plateau was also observed in a specific MDP in Section~5.1 of \citet{cvar_pg_suboptimality}.

In practice, a discrepancy between $q_\alpha^\pi$ and its estimate $\hat{q}_\alpha(\{R(\tau_i)\})$ (used in Equation~\ref{eq:GCVaR_grad}) may prevent the gradient from completely vanishing, if $q_\alpha^\pi = q_\beta^\pi < \hat{q}_\alpha(\{R(\tau_i)\})$.
Otherwise, if $\hat{q}_\alpha(\{R(\tau_i)\}) \le q_\beta^\pi$ in every subsequent iteration, the gradient remains zero, the policy cannot learn any further, and any trajectory returns beyond $q_\alpha^\pi$ will never be even propagated to the optimizer.
We refer to this phenomenon as \textit{blindness to success}.

\begin{definition}[Blindness to success]
\label{def:blindness}
Let a risk level $\alpha\in(0,1)$ and a CVaR-PG training step $m_0 \ge 1$, and let $\beta\in(\alpha,1)$.
Denote by $\mathcal{T},\Pi$ the spaces of trajectories and policies, respectively, and by $\{\tau_{m,i}\}_{i=1}^N \sim P^{\pi_m}$ the random trajectories in step $m\ge m_0$.
We denote by $\mathcal{B}_{\alpha,\beta}^{m_0,n}$ the event of blindness to success in the subsequent $n$ steps (and the complementary event by $\lnot \mathcal{B}_{\alpha,\beta}^{m_0,n}$):
\[\mathcal{B}_{\alpha,\beta}^{m_0,n} = \Big\{ \left\{ \left( \left\{\tau_{m,i}\right\}_{i=1}^N, \, \pi_m \right) \right\}_{m_0\le m<m_0+n} \in (\mathcal{T}^N \times \Pi)^n \ \Big| \ \forall m:\ \hat{q}_\alpha(\{R(\tau_{m,i})\}) \le q_\beta^{\pi_{m_0}} \Big\}.\]
\end{definition}


Note that Definition~\ref{def:blindness} uses $q_\beta^{\pi_{m_0}}$ (corresponding to step $m_0$) to bound the returns in training steps $m>m_0$, thus indeed represents training stagnation.
Theorem~\ref{theorem:blindness_to_success} shows that given a $\beta$-tail barrier with $\beta>\alpha$, the probability that CVaR-PG avoids the blindness to success decreases exponentially with $\beta-\alpha$.
For example, for $n=10^6$, $\alpha=0.05$, $\beta=0.25$, and $N=400$, we have $\mathbb{P}( \lnot \mathcal{B}_{\alpha,\beta}^{m_0,n} ) < 10^{-7}$.

\begin{theorem}
\label{theorem:blindness_to_success}
Under Definition~\ref{def:blindness}'s conditions,
$\mathbb{P} \left( \lnot \mathcal{B}_{\alpha,\beta}^{m_0,n} \,\big|\, \pi_{m_0}\text{ has }\beta\text{-tail barrier} \right) \le ne^{-2N(\beta-\alpha)^2}$.
\end{theorem}

\begin{proof}[Proof sketch (see the full proof in Appendix~\ref{sec:blindness_to_success})]
    In every step $m$, we have $q_\beta^{\pi_{m_0}} < \hat{q}_\alpha(\{R(\tau_{m,i})\})$ only if at least $1-\alpha$ of the returns are higher than $q_\beta^{\pi_{m_0}}$. We bound the probability of this event using the Hoeffding inequality (Lemma~\ref{lemma:good_behaviors}). In the complementary event the gradient is 0 (due to the barrier), thus the policy does not change, and the argument can be applied inductively to the next step.
\end{proof}

\subsection{Variance Reduction and Sample Efficiency}
\label{sec:variance_reduction}

As discussed in Section~\ref{sec:preliminaries}, an MDP can be often re-parameterized as a C-MDP.
In terms of the C-MDP, CVaR-PG samples $N$ context-trajectory pairs from the distribution $P^{\pi_\theta}_{\phi_0}(C,\tau)$, and calculates the policy gradients with respect to the $\alpha N$ trajectories with the lowest returns.
That is, CVaR-PG aims to follow the policy gradients corresponding to the tail distribution defined by
\begin{equation}
    \label{eq:P_alpha}
    P_{\phi_0,\alpha}^{\pi_\theta}(C,\tau) = \alpha^{-1} \pmb{1}_{R(\tau)\le q_\alpha(R|\pi_\theta)} P_{\phi_0}^{\pi_\theta}(C,\tau)
\end{equation}
Notice that by considering only $\alpha$ of the trajectories, CVaR-PG essentially suffers from $\alpha^{-1}$-reduction in sample efficiency in comparison to risk-neutral PG.

Proposition~\ref{prop:variance_reduction} shows that if we could sample trajectories directly from $P_{\phi_0,\alpha}^{\pi_\theta}$, we would reduce the variance of the policy gradient estimate (and thus increase the sample efficiency) back by a factor of $\alpha^{-1}$.
This will motivate the CEM in Section~\ref{sec:method}, which will aim to modify $\phi$ such that $P_{\phi}^{\pi_\theta} \approx P_{\phi_0,\alpha}^{\pi_\theta}$.


\begin{proposition}[Variance reduction]
\label{prop:variance_reduction}
If the quantile estimation error is negligible ($\hat{q}_\alpha = q_\alpha(R|\pi_\theta)$ in Equation~\eqref{eq:GCVaR_grad}), then
$$\mathrm{Var}_{\tau_{i}\sim P_{\phi_0,\alpha}^{\pi_\theta}}(\nabla_\theta \hat{J}_\alpha(\{\tau_{i}\}_{i=1}^{N}; \pi_\theta)) \le \alpha \cdot \mathrm{Var}_{\tau_{i}\sim P_{\phi_0}^{\pi_\theta}}(\nabla_\theta \hat{J}_\alpha(\{\tau_{i}\}_{i=1}^N;\, \pi_\theta)) .$$
\end{proposition}

\begin{proof}[Proof sketch (see the full proof in Appendix~\ref{sec:variance_reduction_proof})]
Since the left term corresponds to the sample distribution $P_{\phi_0,\alpha}^{\pi_\theta}$, the corresponding IS weights are $w\equiv\alpha$ w.p.~1.
When applying IS analysis to the expected value, $w$ cancels out the distributional shift (as in Equation~\ref{eq:CEM0}), resulting in the same expected gradient estimate.
When applying the same analysis to the variance, we begin with the square weight $w^2$, thus a $w=\alpha$ factor still remains after the distributional shift compensation.
\end{proof}


The variance reduction can be connected to sample efficiency through the convergence rate as follows.
According to Theorem 5.5 in \citet{xu2020improved}, denoting the initial parameters by $\theta_0$, the convergence of any CVaR-PG algorithm can be written as 
$\mathbb E[\|\nabla_\theta  J_\alpha(\pi_\theta)\|^2]\leq \mathcal O(\frac{J_\alpha(\theta)-J_\alpha(\theta_0)}{ M}) + \mathcal O(\frac{\mathrm{Var}(\nabla_\theta \hat{J}_\alpha(\{\tau_{i}\}_{i=1}^{N}; \pi_\theta))}{\alpha N})$.
Clearly, variance reduction of $\alpha$-factor linearly improves the second term. In particular, it cancels out the denominator's $\alpha$-factor attributed to tail sub-sampling, and brings the sample efficiency back to the level of the risk-neutral PG.

\section{The Cross-entropy Soft-Risk Algorithm}
\label{sec:method}

Algorithm~\ref{algo:cesor} presents our Cross-entropy Soft-Risk algorithm (\textbf{\textit{CeSoR}}), which uses a PG approach to maximize $J_\alpha(\pi_\theta)$ in~\eqref{eq:cvar}.
CeSoR adds two components on top of CVaR-PG: {\em soft-risk scheduling} to address the blindness to success analyzed in Section~\ref{sec:analysis_blindness}, and {\em CE sampling} to address the sample efficiency analyzed in Section~\ref{sec:variance_reduction}.

\begin{wrapfigure}[27]{i}{0.56\textwidth}
\vspace{-12pt}
\hspace{5pt}
\begin{minipage}{\linewidth}
\begin{algorithm}[H]
\caption{CeSoR}
\label{algo:cesor}
\DontPrintSemicolon
\SetAlgoNoLine
\SetNoFillComment

{\bf Input}: risk level $\alpha$; context distribution $D_{\phi}$; original context parameter $\phi_0$; training steps $M$; trajectories sampled per batch $N$, where $\nu$ fraction of them is from the original $D_{\phi_0}$; smoothed CE quantile $\beta$; risk-level scheduling factor $\rho$\;
 \BlankLine
 {\bf Initialize:} $\;$ policy $\pi_\theta$, $\quad\phi\leftarrow {\phi_0}$,\;
 $N_o\leftarrow \lfloor \nu N \rfloor, \quad $
 $N_s\leftarrow \lceil (1-\nu) N \rceil$\;
 \BlankLine
 \For{$m$ in $1:M$}{
	\tcp{Sample contexts}
	Sample $\{C_{o,i}\}_{i=1}^{N_o} \sim D_{{\phi_0}},\quad\{C_{\phi,i}\}_{i=1}^{N_s} \sim D_{\phi}$\; \label{line:sample}
    $C \leftarrow (C_{o,1},\ldots, C_{o,N_o}, C_{\phi,1},\ldots, C_{\phi,N_s})$\;
    $w_{o,i} \leftarrow 1, \; \forall i\in\{1,\ldots,N_o\}$\; 
    $w_{\phi,i} \leftarrow \frac{D_{\phi_0}(C_{\phi,i})}{ D_{\phi}(C_{\phi,i})}, \; \forall i\in \{1,\ldots,N_s\}$\; \label{line:w_s}
    $w \leftarrow (w_{o,1},\ldots, w_{o,N_o}, w_{\phi,1},\ldots, w_{\phi,N_s})$\; \label{line:w_all}
	\tcp{Sample trajectories}
	$\{\tau_{C_{o,i}}\}, \{\tau_{C_{\phi,i}}\} \leftarrow \text{run\_episodes}(\pi_\theta,\, C)$\; \label{line:run}
	\tcp{Update CE sampler}
	$q \leftarrow \max (\hat{q}_\alpha(\{R(\tau_{C_{o,i}})\}), \hat{q}_{\beta}(\{R(\tau_{C_{\cdot,i}})\}))$\; \label{line:qref}
    $\phi \leftarrow \arg\!\max_{\phi^\prime}\!
    \sum_{i\leq N} \! w_i \;\pmb{1}_{R(\tau_{C_i})\le q} \log\! D_{\phi^\prime}(C_i)$\;
    \label{line:update}
	\tcp{PG step (e.g., Eq.~\ref{eq:GCVaR_grad_is})}
	$\alpha^\prime \leftarrow \max( \alpha,\, 1 - (1-\alpha)\cdot m/(\rho \cdot M) )$\; \label{line:alpha_prime}
	$q^\prime \leftarrow \hat{q}_{\alpha^\prime}(\{R(\tau_{C_{o,i}})\})$\;
	$\theta \leftarrow \text{CVaR\_PG}(\pi_\theta,\, (\{\tau_{C_{o,i}}\}, \{\tau_{C_{\phi,i}}\}),\, w,\, q^\prime)$ \label{line:gcvar}
 }
\end{algorithm}
\end{minipage}
\end{wrapfigure}

\textbf{Soft-risk scheduler}:
We set the policy optimizer (Line~\ref{line:gcvar} in Algorithm~\ref{algo:cesor}) to use a soft risk level $\alpha^\prime$ that gradually decreases from $1$ to $\alpha$ (Line~\ref{line:alpha_prime} and Figure~\ref{fig:soft_risk}).
This is motivated by the blindness to success analyzed in Section~\ref{sec:analysis_blindness}:
by modifying the risk level to $\alpha'> \alpha$, and specifically $\alpha'\approx1$ at the beginning of training, we guarantee that there cannot be a wider tail barrier $\beta > \alpha'$.
Thus, CeSoR can feed the optimizer with trajectories whose returns $q_{\beta}^\pi < R \le q_{\alpha'}^\pi$ are higher than any constant tail; and since the fed returns are not constant, they do not eliminate the gradient. In this sense, CeSoR looks beyond local optimization-plateaus to prevent the blindness to success.



The scheduling defined in Line~\ref{line:alpha_prime} and Figure~\ref{fig:soft_risk} is heuristic.
As demonstrated in Section~\ref{sec:experiments}, once we understand the limitation of blindness to success, this simple heuristic is sufficient to bypass the blindness.
An adaptive $\alpha'$ scheduling that maximizes blindness prevention probability would require tighter concentration inequalities \citep{boucheron2013concentration}, and is left for future work.


\textbf{Cross Entropy Method (CEM)}:
The CEM~\citep{CE_tutorial} is a general approach to rare-event sampling and optimization, which we use to sample high-risk contexts and trajectories.
First, we review the standard CEM in terms adjusted to our setting and notations (for a more general presentation, see Algorithm~\ref{algo:CEM} in the appendix).
Then, we discuss the limitations of the standard CEM in the RL settings, and present our dynamic, regularized version of the CEM.

Motivated by the sample efficiency analysis of Section~\ref{sec:variance_reduction}, we wish to align the agent's experience with the $\alpha$ worst-case returns -- by sampling contexts whose corresponding trajectory-returns are likely to be below $q_\alpha(R|\pi_\theta)$.
That is, we wish to sample context-trajectory pairs from $P_{\phi_0,\alpha}^{\pi_\theta}$ of \eqref{eq:P_alpha}.
To that end, the CEM searches for a value of $\phi$ for which $P_{\phi}^{\pi_\theta}$ is similar to $P_{\phi_0,\alpha}^{\pi_\theta}$.
More precisely, it looks for $\phi^*$ that minimizes the KL-divergence (i.e., cross-entropy) between the two:
%
%
\begin{align}
\label{eq:CEM0}
\begin{split}
\phi^* &\in \mathrm{argmin}_{\phi'} \; D_{KL}\big(P_{\phi_0,\alpha}^{\pi_\theta}(C,\tau) \,||\, P_{\phi'}^{\pi_\theta}(C,\tau)\big) \\ 
&= \ \mathrm{argmax}_{\phi'} \; \mathbb{E}_{(C,\tau)\sim P_{\phi_0}^{\pi_\theta}} \big[\alpha^{-1} \pmb{1}_{R(\tau)\le q_\alpha(R|\pi_\theta)} \log D_{\phi'}(C)\big] \\
&= \ \mathrm{argmax}_{\phi'} \; \mathbb{E}_{(C,\tau)\sim P_{\phi}^{\pi_\theta}} \big[\alpha^{-1} w(C,\tau)\;\pmb{1}_{R(\tau)\le q_\alpha(R|\pi_\theta)} \log D_{\phi'}(C)\big],
\end{split}
\end{align}
%
where $P^{\pi_\theta}_{\phi'}(C,\tau) = D_{\phi'}(C)P_C^{\pi_\theta}(\tau)$ (Section~\ref{sec:preliminaries}), and $w(C,\tau)=\frac{P^{\pi_\theta}_{\phi_0}(C,\tau)}{P^{\pi_\theta}_{\phi}(C,\tau)} = \frac{D_{\phi_0}(C)}{D_{\phi}(C)}$ is the IS weight corresponding to the sample distribution $(C,\tau) \sim P_{\phi}^{\pi_\theta}$.
The optimization problem in Equation~\eqref{eq:CEM0} often reduces to a simple closed-form calculation: if $D_\phi$ is a Gaussian, for example, $\phi^*$ reduces to the weighted expectation and variance of $\{C \,|\, R(\tau) \le q_\alpha\}_{C,\tau \sim P^{\pi_\theta}_{\phi}}$ with the IS weights $w(C,\tau)$.

Equation~\eqref{eq:CEM0} may produce noisy results when estimated from data $\{(C_i,\tau_i)\}_{i=1}^N$, unless $N \gg \alpha^{-1}$, since only $\alpha N$ trajectory-samples satisfy $R(\tau) \le q_\alpha$ and are used in the estimation.
To address this, the CEM reaches the $\alpha$-tail gradually over iterations.
Every iteration, it samples a batch of contexts $\{C_{i}\}_{i=1}^{N}$ from the current distribution $D_\phi$, and then solves Equation~\eqref{eq:CEM0} with respect to a \textit{higher} quantile $q\ge q_\alpha$.
More specifically, denote by $\hat{q}_\alpha^{\phi}$ the estimated $\alpha$-quantile of $\{R(\tau)\}_{C,\tau \sim P^{\pi_\theta}_{\phi}}$; then, we set $q=\max(\hat{q}_\alpha^{\phi_0},\, \hat{q}_\beta^{\phi})$ with a hyperparameter $\beta>\alpha$ (often $\beta=0.2$).
Since the data is drawn from $P^{\pi_\theta}_{\phi}$, this guarantees at least $\beta N$ samples per update step.
The quantile $\hat{q}_\alpha^{\phi_0}$ corresponds to the $\alpha$-tail of the original context-distribution, and can be viewed as a stopping condition: once $\hat{q}_\alpha^{\phi_0} > \hat{q}_\beta^{\phi}$, many of our samples are already in the tail, and $\beta$ is no longer needed to smooth the update of $\phi$.


\begin{figure}[!t]
\centering
\begin{minipage}{0.35\textwidth}
    \centering
    \includegraphics[width=1.\linewidth]{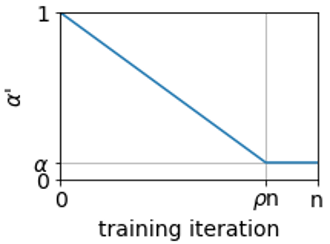}
    \caption{\small The soft-risk scheduling (Algorithm~\ref{algo:cesor}, Line~\ref{line:alpha_prime}). The linear phase $\alpha'>\alpha$ prevents the blindness to success (Section~\ref{sec:analysis_blindness}), while the CEM still preserves risk aversion. The final constant phase $\alpha'=\alpha$ provides a stationary objective and allows CeSoR to converge (Appendix~\ref{sec:analysis_gradient}).}
    \label{fig:soft_risk}
\end{minipage}%
\hspace{0.05\textwidth}
\begin{minipage}{0.59\textwidth}
    \centering
    \includegraphics[width=1.\linewidth]{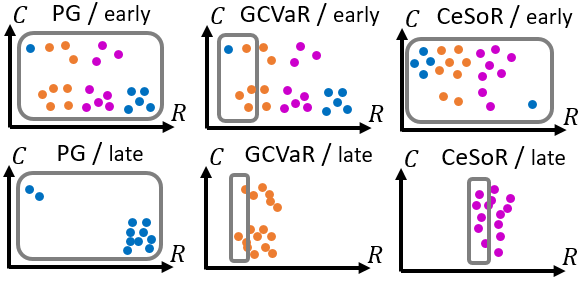}
    \caption{\small An illustration of training batches. Each point represents an episode with return $R$ and context $C$. Points of the same color correspond to “similar” agent actions that induce similar policy gradients.
    Mean-PG averages over the \textit{whole} batch and learns the blue strategy. CVaR-PG considers the \textit{left} part (low returns) and learns the orange strategy. CeSoR over-samples the \textit{upper} part (high-risk contexts), and only later decreases $\alpha^\prime$ to explicitly focus on low returns, thus learning the purple strategy. 
    The illustrated episodes are analogous to the strategies in Figures~\ref{fig:maze_nice},\ref{fig:maze_exposure}.}
    \label{fig:cesor_illustration}
\end{minipage}
\end{figure}


\textbf{Dynamic-target CEM}:
The standard CEM assumes to search for the tail of a \textit{constant} distribution.
In our setting, however, we look for the tail of the distribution of the returns $R(\tau)$, where $C,\tau \sim P^{\pi_\theta}_{\phi_0}$ depend on $\pi_\theta$ and thus are non-stationary throughout the training.
The non-stationarity poses several challenges for the CEM.
First, the stopping condition $\hat{q}_\alpha^{\phi_0}$ varies with $\pi_\theta$ and has to be re-estimated every iteration\footnote{For the sake of coherent notations, we presented the CEM with the quantile objective $q_\alpha(R|\pi_\theta)$. In fact, the standard CEM is usually defined with a constant numeric objective $q_0\in\mathbb{R}$ rather than a quantile; hence, as shown in Algorithm~\ref{algo:CEM} in the appendix, the standard CEM does not require any quantile estimation at all.}.
Second, the high-risk contexts $C$ (which correspond to the lowest returns) may vary as the agent evolves; and if the CEM learns to only sample a strict subset of the context space, then it may miss such changes in the high-risk contexts.

We address both issues using reference samples: every iteration, we sample \textit{two} batches of contexts -- $\{C_{\phi,i}\}_{i=1}^{N_s}$ from the current context distribution $D_\phi$ and $\{C_{o,i}\}_{i=1}^{N_o}$ from the original distribution $D_{\phi_0}$.
The reference contexts provide an important regularization: they guarantee continual exposure to the whole context space, in case that the high-risk contexts vary.
In addition, the reference samples were empirically found to stabilize the estimation of $\hat{q}_\alpha^{\phi_0}$ (Line~\ref{line:qref} in Algorithm~\ref{algo:cesor}).

Consider the two batches of context-trajectory pairs, and denote the estimated return quantile $\hat{q}_\alpha=\hat{q}_\alpha(\{R(\tau_i)\}_{i=1}^{N_o})$.
We can estimate the CVaR policy gradient, using the notation
$\forall 1\le i\le N_o+N_s:\ C_i = \begin{cases} C_{o,i} & \text{if } 1\le i\le N_o \\ C_{\phi,i-N_o} & \text{if } N_o+1\le i\le N_o+N_s \end{cases}$, by
\begin{align}
\label{eq:GCVaR_grad_is}
\nabla_\theta \hat{J}_\alpha(\pi_\theta) = \frac{1}{\alpha (N_o+N_s)} \sum_{i=1}^{N_o+N_s} w_i\cdot \pmb{1}_{R(\tau_i)\le \hat{q}_\alpha} \left(R(\tau_i) - \hat{q}_\alpha\right) \sum_{t=0}^T\!\nabla_\theta\log \pi_\theta(a_{i,t};s_{i,t}),
\end{align}
where $w_i=1$ for $1\leq i \leq N_o$ and $w_i=D_{\phi_0}(C_i)/D_{\phi^*}(C_i)$ for $N_o+1\leq i\leq N_o+N_s$.

Note that if the policy learning scale is slower than that of $\phi$, the target context distribution $P^{\pi_\theta}_{\phi_0,\alpha}$ is effectively stationary in the $\phi$-optimization problem. In that case, according to \citet{RareEE}, the CEM will converge to the KL-divergence minimizer $\phi^*$ of~\eqref{eq:CEM0}.

\textbf{Sample efficiency in practice}:
Proposition~\ref{prop:variance_reduction} guarantees an $\alpha^{-1}$-increase in sample efficiency when using an accurate quantile estimate $\hat{q}_\alpha = q_\alpha(R|\pi_\theta)$ and sampling exactly from $P_{\phi_0,\alpha}^{\pi_\theta}$.
The latter condition is equivalent to the CE-sampler reaching its objective $D_{KL}( P_{\phi_0,\alpha}^{\pi_\theta} \,||\, P_{\phi}^{\pi_\theta}) = 0$.
In practice, $P_{\phi_0,\alpha}^{\pi_\theta}$ can only be approximated, and the sample efficiency is increased -- but by a smaller factor than $\alpha^{-1}$.
Appendix~\ref{sec:ce_sample_efficiency} demonstrates the increased sample size exploited by CeSoR in our experiments.

If $\hat{q}_\alpha \ne q_\alpha(R|\pi_\theta)$, the quantile estimation error may theoretically lead to unbounded IS weights (see Appendix~\ref{sec:variance_reduction_proof}). Practically, we address this by clipping the weights (as mentioned in Section~\ref{sec:experiments}), and by constraining the family of permitted distributions $\{D_\phi\}_\phi$ to have a constant support independently of $\phi$.
A side-effect is a function approximation error of the family $\{D_\phi\}_\phi$, as $D_{\phi^*}(C)P^{\pi_\theta}_{C}(\tau)$ cannot replicate the tail distribution $P^{\pi_\theta}_{\phi_0,\alpha}(C,\tau)$ to achieve the full $\alpha^{-1}$-increase in sample efficiency.

Another limitation in the expressiveness of $P^{\pi_\theta}_{\phi}(C,\tau) = D_{\phi}(C)P^{\pi_\theta}_{C}(\tau)$ occurs when the context $C$ only controls part of the environment randomness in $P^{\pi_\theta}_{C}(\tau)$.
As an extreme example in the Guarded Maze, after $\pi_\theta$ already learns to avoid the short path, the context (guard cost) does not affect the outcome at all anymore.
Indeed, Figure~\ref{fig:ce_maze_dist} in the appendix shows that high guard costs are sampled in the beginning; then, once the short path is avoided, the sampler gradually falls back to the original context distribution.
Note that in this example, the invariance to $C$ began after the learning was essentially done, hence the CEM did play its part effectively.

Finally, note that the soft risk creates an intentional bias in the gradient estimate (to overcome the blindness to success).
As a result, in the first phase of training ($\alpha'\gg\alpha$), only a few trajectories are overlooked every iteration.
As $\alpha'$ approaches $\alpha$, the number of overlooked trajectories increases, and so is the importance of over-sampling the tail.
In the final steady-state phase ($\alpha'=\alpha$), the sample inefficiency is most severe, the soft risk produces no further biases, and the CEM helps CeSoR to reduce the high variance in the policy gradient estimation.

\textbf{The harmony between the soft risk and the CEM}:
Soft risk has the inherent side effect of reducing the risk aversion.
In the Guarded Maze, for example, as demonstrated in Section~\ref{sec:maze}, soft risk alone leads to learning the short path (instead of the risk-averse long path).
Fortunately, the CEM reduces this side effect.
In that sense, the two mechanisms complement each other: $\alpha'>\alpha$ allows the \textit{optimizer} to learn policies with high returns, while the CE \textit{sampler} still preserves the risk aversion -- as illustrated in Figure~\ref{fig:cesor_illustration}.
This connection stands in addition to the independent motivations of the two mechanisms, as discussed above.

\textbf{Baseline optimizer}:
CeSoR can be implemented on top of any CVaR-PG method as a baseline (Line~\ref{line:gcvar}).
We use the standard GCVaR~\citep{cvar_via_sampling}, which guarantees asymptotic convergence under certain regularity conditions. Appendix~\ref{sec:analysis_gradient} shows that these guarantees hold for CeSoR as well, when implemented on top of GCVaR.
Other CVaR-PG baselines can also be used, such as the TRPO-based algorithm of \citet{cvar_trpo}. However, such methods often include heuristics that introduce additional gradient estimation bias (to reduce variance), and thus do not necessarily guarantee the same theoretical convergence.




\section{Experiments}
\label{sec:experiments}

We conduct experiments in 3 different domains.
We implement \textbf{CeSoR} on top of a standard CVaR-PG method, which is also used as a risk-averse baseline for comparison. Specifically, we use the standard \textbf{GCVaR} baseline~\citep{cvar_via_sampling}, which guarantees convenient convergence properties (see Appendix~\ref{sec:analysis_gradient}) and is simple to implement and analyze.
We also use the standard policy gradient (\textbf{PG}) as a risk-neutral baseline. We stress that the comparison to PG is only intended to present the mean-CVaR tradeoff, while each method legitimately optimizes its own objective. 
Appendix~\ref{sec:dRL} also compares CeSoR to risk-neutral and risk-averse Distributional RL algorithms.



In all the experiments, all agents are trained using Adam~\citep{adam}, with a learning rate selected manually per benchmark and $N=400$ episodes per training step.
Every 10 steps we run validation episodes, and we choose the final policy according to the best validation score (best mean for PG, best CVaR for GCVaR and CeSoR).
For CeSoR, unless specified otherwise, 
$\nu=20\%$ of the trajectories per batch are drawn from the original distribution $D_{\phi_0}$; $\beta=20\%$ are used for the CE update; and 
the soft risk level reaches $\alpha$ after $\rho=80\%$ of the training.
As mentioned in Section~\ref{sec:method}, for numerical stability, we also clip the IS weights (Algorithm~\ref{algo:cesor}, Line~\ref{line:w_all}) to the range $[1/5,5]$.

Every policy is modeled as a neural network with $tanh$ activation on its middle layers and $softmax$ operator on its output, with temperature $1$ in training (i.e., network outputs are actions probabilities), and $0$ in validation and test (i.e., the max output is the selected action).
We use a middle layer with 32 neurons in Section~\ref{sec:driver}, 16 neurons in Section~\ref{sec:servers}, and no middle layer (linear model) in Section~\ref{sec:maze}.

In each of the 3 domains, the experiments required a running time of a few hours on an Ubuntu machine with eight i9-10900X CPU cores.
In addition to these RL-related experiments, Appendix~\ref{sec:CE_results} presents dedicated experiments for the independent CE module.


\begin{wrapfigure}[18]{R}{0.47\textwidth}
\vspace{-43pt}
\centering
  \includegraphics[width=1.\linewidth]{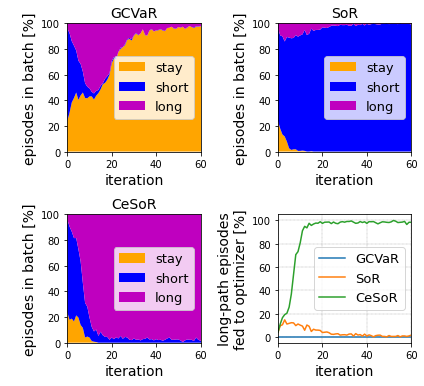}
\caption{\small GCVaR, SoR, CeSoR: $\%$-episodes that did not reach the target ("stay"), or reached it through the short or the long path in the Guarded Maze. Bottom Right: $\%$-long-paths among the trajectories fed to the optimizer. See more details in Figure~\ref{fig:blindness_to_success}.}
\label{fig:maze_exposure}
\end{wrapfigure}

\subsection{The Guarded Maze}
\label{sec:maze}

\textbf{Benchmark:}
The Guarded Maze benchmark is defined in Section~\ref{sec:intro}.
For the experiments, we set a target risk level of $\alpha=0.05$, and train each agent for $n=250$ steps with the parameters described above.
The CEM controls $C$ through $\phi=(\phi_1,\phi_2)$, where $\phi_0=(0.2,32)$ as mentioned above, and updates $\phi_1,\phi_2$ using the weighted means of $C_1$ and $C_2$, respectively.
As an ablation test, we add two partial variants of our CeSoR: \textbf{CeR} (with CE, without $\alpha$-scheduling) and \textbf{SoR} (with scheduling, without CE).
See more details in Appendix~\ref{sec:maze_implementation}.

\textbf{Results:}
Figure~\ref{fig:maze_test_scores} summarizes the test scores, and Figure~\ref{fig:maze_nice} illustrates a sample episode.
PG learned the short path, maximizing the average but at the cost of poor returns whenever charged by the guard.
CeSoR, on the other hand, successfully learned to follow the CVaR-optimal long path.
GCVaR, which also aimed to maximize the CVaR, failed to do so. As analyzed in Figure~\ref{fig:maze_exposure}, throughout GCVaR training, the agent takes the long path in up to 50\% of the episodes per batch, but none of these episodes is ever included in the bottom $\alpha=5\%$ that are fed to the optimizer. 
Thus, GCVaR is entirely \textit{blind} to the successful episodes and fails to learn the corresponding strategy.
In fact, in most training steps, \textit{all} the worst episodes of GCVaR reach neither the guard nor the target, leading to a constant return of $-32$, a tail barrier, and a zero loss-gradient. 

CeR suffers from blindness to success just as GCVaR.
SoR is exposed to the successful long-path episodes thanks to soft risk scheduling; however, due to the reduced risk-aversion, it fails to prefer the long path over the short one.
Only CeSoR both observes the \textit{"good" strategy} (thanks to soft risk scheduling) and judges it under \textit{"bad" environment variations} (thanks to the CEM).
Appendix~\ref{sec:maze_detailed_results} presents a detailed analysis of the learning dynamics, the blindness to success and the learned policies.
It is important to notice that standard optimization tweaks cannot bring GCVaR to learn the long path: a "warm-start" from a standard PG only encourages the short-path policy (as in SoR); and increased batch size $N$ does not expose the optimizer to the long path (see Theorem~\ref{theorem:blindness_to_success}).


\subsection{The Driving Game}
\label{sec:driver}

\textbf{Benchmark:}
The Driving Game is based on an inverse-RL benchmark used by \citet{driving2017} and \citet{driving2018}.
The agent's car has to follow the leader (an "erratic driver") for 30 seconds as closely as possible without colliding.
Every 1.5 seconds (i.e., 20 times per episode), the leader chooses a random action (independently of the agent): drive straight, accelerate, decelerate, change lane, or brake hard ("emergency brake"), with respective probabilities $\phi_0=(0.35,0.3,0.248,0.1,0.002)$. We denote the sequence of leader actions by $C\in\{1,...,5\}^{20}$.

\begin{wrapfigure}[27]{R}{0.4\textwidth}
\vspace{-13pt}
\centering
\begin{subfigure}{1.\linewidth}
  \centering
  \includegraphics[width=1.\linewidth]{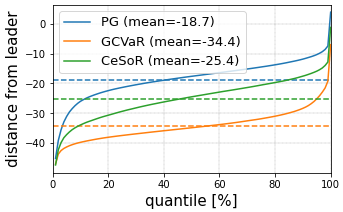}
  \caption{}
  \label{fig:driving_actions}
\end{subfigure}
\begin{subfigure}{.4\textwidth}
  \centering
  \includegraphics[width=1.\linewidth]{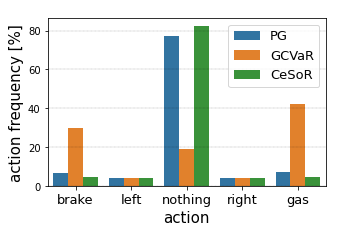}
  \caption{}
  \label{fig:driving_dx}
\end{subfigure}
\caption{\small Over all the time-steps in all the test episodes in the Driving Game, the distribution of (a) the distance between the agent and the leader, (b) the agent actions. Evidently, CeSoR learns to keep more distance than the risk-neutral PG, and has a slightly less frequent use of the gas and the brake.}
\label{fig:driving_behavior}
\end{wrapfigure}

Every 0.5 seconds (60 times per episode), the agent observes its relative position and velocity to the leader, with a delay of 0.7 seconds (representing reaction time), as well as its own acceleration and steering direction. The agent chooses one of the five actions: drive in the same steering direction, accelerate, decelerate, turn left, or turn right. Changing lane is not an atomic action and has to be learned using turns.
The rewards express the requirements to stay behind the leader, on the road, on the same lane, not too far behind and without colliding.
See the complete details in Appendix~\ref{sec:driver_implementation}.

We set $\alpha=0.01$, and train each agent for $n=500$ steps.
To initiate learning, for each agent we begin with shorter training episodes of 6 seconds and gradually increase their length. 
The CEM controls the leader's behavior through the probabilities $\phi=\{\phi_i\}_{i=1}^5$ described above.


\textbf{Results:}
Figure~\ref{fig:driving_test_scores} summarizes the test scores of the agents, where CeSoR presents a reduction of 28\% in the CVaR cost in comparison to the baselines. GCVaR completely fails to learn a reasonable policy -- losing in terms of CVaR even to the risk-neutral PG.
Figure~\ref{fig:driving_behavior} shows that CeSoR learned an arguably-intuitive policy for risk averse driving: it keeps a safer distance from the leader, and uses the gas and the brake less frequently.
This results in complete avoidance of the rare accidents occurring to PG, as demonstrated in Figure~\ref{fig:driving_nice}.
In Appendix~\ref{sec:CE_results}, we also see that by over-sampling turns and emergency brakes of the leader, the CEM manages to align the mean return of the training samples with the 1\%-CVaR of the environment, and significantly increases the data efficiency.


\subsection{Computational Resource Allocation}
\label{sec:servers}

\textbf{Benchmark:}
Computational resource allocation in serving systems, and in particular the tradeoff between resource cost and serving latency, is an important challenge to both academia~\citep{cloud_resource_scaling,chen_datacenters} and industry~\citep{aws_predictive_scaling,granulate_techcrunch2}.
In popular applications such as E-commerce and news, latency is most critical at times of peak loads~\citep{peak_loads}, making CVaR a natural metric for risk-averse optimization.
In our benchmark, the agent allocates servers to handle user requests, managing the tradeoff between servers cost and time-to-service (TTS).
Requests arrive randomly with a constant rate, up to rare events that cause sudden peak loads, whose frequency is controlled by the CE sampler.
See Appendix~\ref{sec:servers_detailed} for more details.

\textbf{Results:} As shown in Figure~\ref{fig:servers_test_scores}, CeSoR significantly improves the CVaR return, and does not compromise the mean as much as GCVaR.
As demonstrated in Figure~\ref{fig:servers_nice}, CeSoR learned to allocate a default of 5 servers and react to peak loads as needed, whereas GCVaR simply allocates 8 servers at all times.
PG only allocates 4 servers by default, and thus its TTS is more sensitive to peak loads.
Appendix~\ref{sec:servers_detailed} describes the complete implementation and detailed results, discusses the poor parameterization of $D_\phi$ in this problem and shows the robustness of CeSoR to that parameterization.

\section{Summary and Future Work}
\label{sec:summary}

We introduced CeSoR, a novel method for risk-averse RL, focused on efficient sampling and soft risk.
In a variety of experimental domains, in comparison to a risk-averse baseline, CeSoR demonstrated higher CVaR metric, better sample-efficiency, and elimination of blindness to success -- where the latter two were also analyzed theoretically.

There are certain limitations to CeSoR.
First, we assume to have at least partial control over the training conditions, through a parametric family of distributions that needs to be selected.
Second, CeSoR can be applied robustly on top of any CVaR-PG method, but is currently not applicable to non-PG methods. Since the limitations of CVaR-PG apply in other risk-averse methods as well (as we demonstrated for Distributional RL), future work may adjust CeSoR to such methods, as well as to other risk measures.
Third, in terms of blindness to success and estimation variance, CeSoR shows both theoretical and empirical improvement -- but is not proven optimal. Future work may look for optimal design of CEM or risk scheduling.
Considering the current results and the potential extensions, we believe CeSoR may open the door for more practical applications of risk-averse RL.


\textbf{Acknowledgements:}
This research was supported by the Israel Science Foundation (grant 2199/20).



\bibliographystyle{plainnat}
\bibliography{main}


\newpage

\section*{Checklist}

\begin{enumerate}

\item For all authors...
\begin{enumerate}
  \item Do the main claims made in the abstract and introduction accurately reflect the paper's contributions and scope?
    \answerYes{}
  \item Did you describe the limitations of your work?
    \answerYes{See the summary (Section~\ref{sec:summary}).}
  \item Did you discuss any potential negative societal impacts of your work?
    \answerNA{Risk sensitive RL is an abstract task and is not specifically associated with any negative-impact applications.}
  \item Have you read the ethics review guidelines and ensured that your paper conforms to them?
    \answerYes{}
\end{enumerate}

\item If you are including theoretical results...
\begin{enumerate}
  \item Did you state the full set of assumptions of all theoretical results?
    \answerYes{For each result, Section~\ref{sec:analysis} either states the assumptions directly or refers to Appendices~\ref{sec:blindness_to_success}-\ref{sec:analysis_gradient}.}
\item Did you include complete proofs of all theoretical results?
    \answerYes{See Section~\ref{sec:analysis} with references to Appendices~\ref{sec:blindness_to_success}-\ref{sec:analysis_gradient}.}
\end{enumerate}

\item If you ran experiments...
\begin{enumerate}
  \item Did you include the code, data, and instructions needed to reproduce the main experimental results (either in the supplemental material or as a URL)?
    \answerYes{A \repo{repository} of our code presents Jupyter notebooks that reproduce the experimental results.}
  \item Did you specify all the training details (e.g., data splits, hyperparameters, how they were chosen)?
    \answerYes{General details in the beginning of Section~\ref{sec:experiments}, and specific details per experiment in Sections~\ref{sec:maze}-\ref{sec:servers}.}
\item Did you report error bars (e.g., with respect to the random seed after running experiments multiple times)?
    \answerYes{In Figures~\ref{fig:awesome},\ref{fig:maze_test_scores_full},\ref{fig:driving_test_scores_full},\ref{fig:servers_test_scores_full}, we provide not only the average return but the full distribution, which implicitly includes the information of error bars.}
        \item Did you include the total amount of compute and the type of resources used (e.g., type of GPUs, internal cluster, or cloud provider)?
    \answerYes{In the beginning of Section~\ref{sec:experiments}.}
\end{enumerate}

\item If you are using existing assets (e.g., code, data, models) or curating/releasing new assets...
\begin{enumerate}
  \item If your work uses existing assets, did you cite the creators?
    \answerNA{}
  \item Did you mention the license of the assets?
    \answerNA{}
  \item Did you include any new assets either in the supplemental material or as a URL?
    \answerYes{An anonymized repository of our code is linked from the abstract.}
  \item Did you discuss whether and how consent was obtained from people whose data you're using/curating?
    \answerNA{}
  \item Did you discuss whether the data you are using/curating contains personally identifiable information or offensive content?
    \answerNA{}
\end{enumerate}

\item If you used crowdsourcing or conducted research with human subjects...
\begin{enumerate}
  \item Did you include the full text of instructions given to participants and screenshots, if applicable?
    \answerNA{}
  \item Did you describe any potential participant risks, with links to Institutional Review Board (IRB) approvals, if applicable?
    \answerNA{}
  \item Did you include the estimated hourly wage paid to participants and the total amount spent on participant compensation?
    \answerNA{}
\end{enumerate}

\end{enumerate}


\appendix
\newpage


\setcounter{tocdepth}{1}
\tableofcontents
\newpage


\section{Blindness to Success: Proof of Theorem \ref{theorem:blindness_to_success}}
\label{sec:blindness_to_success}

Theorem~\ref{theorem:blindness_to_success} considers the probabilistic event of a global blindness to success over $n$ consecutive training steps.
We begin with a local blindness in a single training step.

\begin{lemma}[Local blindness to success]
\label{lemma:good_behaviors}
Let a risk level $\alpha\in(0,1)$ and a CVaR-PG training step $m \ge 1$, and let $\beta \in (\alpha, 1)$. Denote $A = \left\{ \left\{\tau_{m,i}\right\}_{i=1}^N \in \mathcal{T}^N \ \big|\ q_\beta^{\pi_m} < \hat{q}_\alpha(\{R(\tau_{m,i})\}_{i=1}^N) \right\}$.
Then,
$$\mathbb{P} \left( A \right) \le e^{-\frac{N(\beta-\alpha)^2}{2\beta(1-\beta)}} \le e^{-2N(\beta-\alpha)^2}$$
\end{lemma}

\begin{proof}
Denote $R_i=R(\tau_{m,i})$, $\ \chi_i^q = \pmb{1}_{R_i > q}$, and $\chi_i = \chi_i^{q_\beta^{\pi_m}}$.
Note that $\chi_i \sim Bernoulli(1-\beta)$. Also denote the percent of high-return trajectories by $n_q=\sum_{i=1}^N \chi_i^q / N$ and $n^* = n_{q_\beta^\pi}$.
Since $\hat{q}_\alpha(\{R_i\}_{i=1}^N) = \min \left\{ q \, | \, \frac{\left| \{ i \, | \, R_i\le q \} \right|}{N} \ge \alpha \right\}  = \min \left\{q \, | \, \frac{\left| \{ i \, | \, R_i> q \} \right|}{N} < 1-\alpha \right\} = \min \left\{q \, | \, \frac{1}{N}\sum_{i=1}^N \chi_i^q < 1-\alpha \right\}$,
we have $q_\beta^\pi < \hat{q}_\alpha(\{R_i\}) \Leftrightarrow n^* \ge 1-\alpha$, i.e., $A = \left( n^* \ge 1-\alpha \right)$.


Since $\mathbb{P}(0\le\chi_i\le1)=1$, $E\left[\chi_i\right] = 1-\beta$ and
the Bernoulli $\chi_i$ are sub-Gaussian with variance factor $\sigma^2=1/4$, by Hoeffding inequality~\citep{hoeffding} we obtain
$$ \mathbb{P}(A) = \mathbb{P}(n^* \ge 1-\alpha) = \mathbb{P}(n^*-E\left[n^*\right] \ge \beta-\alpha) \le e^{-\frac{N^2(\beta-\alpha)^2}{2\sum_i 1/4}} = e^{-2N(\beta-\alpha)^2} . $$
\end{proof}

Note that Lemma~\ref{lemma:good_behaviors} does not depend on a tail-barrier: it simply implies that since a CVaR-PG algorithm focuses on the worst $\alpha$ trajectories in every batch, we do not expect trajectories with high returns $R(\tau_{m,i}) > q_\beta^{\pi_m}$ to be fed to the optimizer.
Still, in general, even if high-return trajectories are ignored, the CVaR-PG can learn to avoid low-return trajectories with $R(\tau_{m,i}) < q_\beta^{\pi_m}$.
The tail barrier prevents this learning, since there are no returns strictly lower than $q_\beta^{\pi_m}$ -- all the tail identically equals $q_\beta^{\pi_m}$.
Since there are no worse trajectories to learn from, and better trajectories are ignored, this brings the training to a deadlock, as stated by Theorem~\ref{theorem:blindness_to_success}.

\begin{proof}[Proof of Theorem~\ref{theorem:blindness_to_success} (stated in Section~\ref{sec:analysis_blindness})]
All probabilities below are conditioned on the event of $\pi_{m_0}$ having a $\beta$-tail barrier. Thus, we simplify the notation to $\mathbb{P}(\cdot) = \mathbb{P} \left( \cdot \,\big|\, \pi_{m_0}\text{ has }\beta\text{-tail barrier} \right)$.

Denote by $\mathcal{S} = \left\{ \left( \left\{\tau_{m,i}\right\}_{i=1}^N, \, \pi_m \right) \right\}_{m=m_0}^{m_0+n-1} \in (\mathcal{T}^N \times \Pi)^n$ the sequence of trajectory batches and policies, and by $R_m=\{R(\tau_{m,i})\}_{i=1}^N$ the returns on step $m$. Also denote for simplicity $\mathcal{B}=\mathcal{B}_{\alpha,\beta}^{m_0,n}$.
We are interested in the probability of the event that there is no global blindness (Definition~\ref{def:blindness}):
$$ \lnot \mathcal{B} = \lnot \mathcal{B}_{\alpha,\beta}^{m_0,n} = \left\{ \mathcal{S} \ \big| \ \exists m_0\le m< m_0+n:\, q_\beta^{\pi_{m_0}} < \hat{q}_\alpha(R_m) \right\}. $$
Define the event of blindness at step $m$, along with an unchanged policy: $A_m = \left\{ \mathcal{S} \,\big|\, \pi_{m}=\pi_{m_0} \land \hat{q}_\alpha(R_m) \le q_\beta^{\pi_{m_0}} \right\}$.
Note that $\bigcap_{m=m_0}^{m_0+n-1}A_m \subseteq \mathcal{B}$, hence $$\mathbb{P}(\lnot\mathcal{B}) \le 1-\mathbb{P}\left( \bigcap_{m=m_0}^{m_0+n-1}A_m \right) = 1-\prod_{m=m_0}^{m_0+n-1}\mathbb{P}\left( A_m | A_{m_0},...,A_{m-1} \right) .$$
Thus, to complete the proof, we show below that $\mathbb{P}\left( A_m | A_{m_0},...,A_{m-1} \right) \ge 1-\delta$, where $\delta = e^{-2N(\beta-\alpha)^2}$, hence $\mathbb{P}(\lnot\mathcal{B}) \le 1-(1-\delta)^n \le 1-(1-n\delta) = n\delta$.

For $m=m_0$, we have immediately $\pi_m=\pi_{m_0}$, and from Lemma~\ref{lemma:good_behaviors} $\mathbb{P}(q_\beta^{\pi_{m_0}} < \hat{q}_\alpha(R_{m_0})) \le \delta$.
For $m_0+1 \le m \le m_0+n-1$, assume that $A_{m_0},...,A_{m-1}$ hold.
In particular, $\hat{q}_\alpha(R_{m-1}) \le q_\beta^{\pi_{m_0}}$, $\pi_{m-1}=\pi_{m_0}$ and $\pi_{m-1}$ has a $\beta$-tail barrier.
Now consider the $m-1$ training batch: for every trajectory $1\le i\le N$, if $R_{m-1,i}>\hat{q}_\alpha(R_{m-1})$, then $\pmb{1}_{R_{m-1,i}\le \hat{q}_\alpha(R_{m-1})}=0$; otherwise, $R_{m-1,i} \le \hat{q}_\alpha(R_{m-1}) \le q_\beta^{\pi_{m_0}}$, that is, $R_{m-1,i} = q_{\beta^\prime}^{\pi_{m_0}}$ for some $\beta^\prime\le\beta$, and by the barrier property $R_{m-1,i} = q_{\beta}^{\pi_{m_0}}$ and thus $R_{m-1,i}-\hat{q}_\alpha(R_{m-1})=0$.
Hence, the gradient in Equation~\eqref{eq:GCVaR_grad} is 0, the policy update vanishes, and we obtain $\pi_m=\pi_{m-1}=\pi_{m_0}$.
Then again, according to Lemma~\ref{lemma:good_behaviors} (and since $R_{m}$ and $R_{m_0}$ are drawn from the same distribution corresponding to $\pi_m=\pi_{m_0}$), we have $\mathbb{P}(q_\beta^{\pi_{m_0}} < \hat{q}_\alpha(R_{m})) = \mathbb{P}(q_\beta^{\pi_{m_0}} < \hat{q}_\alpha(R_{m_0})) \le \delta$, as required.
\end{proof}

Note that the factor $n$ may become quite negligible when the barrier is wider than $\alpha$: if $n=10^6, \alpha=0.05, \beta=0.25, N=400$, for example, we still have $\mathbb{P}(\lnot \mathcal{B}_{\alpha,\beta}^{m_0,n}) < 10^{-7}$.
Indeed, the blindness occurs with significantly smaller barriers than the $\beta=0.9$ demonstrated in the Guarded Maze in Appendix~\ref{sec:maze_detailed_results}.
Note that the momentum term of the Adam algorithm~\citep{adam}, while preventing the policy update from completely vanishing, was empirically insufficient to overcome the barrier in the Guarded Maze. This should not come as a surprise, since the momentum comes from previous gradients that \textit{encouraged} the strategies of the barrier and brought them into the tail in the first place.

\section{Variance Reduction: Proof of Proposition \ref{prop:variance_reduction}}
\label{sec:variance_reduction_proof}

\begin{proof}
Define $H(C, \tau) = \alpha^{-1} \pmb{1}_{R(\tau)\le \hat{q}_\alpha} \left(R(\tau)-\hat{q}_\alpha\right) \nabla_\theta\sum_t \log \pi_\theta(a_t;s_t)$, 
such that the CVaR PG can be written as $\nabla_\theta \hat{J}_\alpha \left( \{C_i,\tau_i\}_{i=1}^N;\pi_\theta \right) = \frac{1}{N}\sum_{i=1}^N w(C_i,\tau_i) H(C_i,\tau_i)$, where $w(C,\tau) = \frac{P^{\pi_\theta}_{\phi_0}(C,\tau)}{P^{\pi_\theta}_{\phi_0,\alpha}(C,\tau)}$ is the IS weighting that accounts for the modified sample distribution. 
Since $C,\tau \sim P^{\pi_\theta}_{\phi_0,\alpha}$, we have $R(\tau)\le q_\alpha(R|\pi_\theta)$ almost surely; and along with the assumption $\hat{q}_\alpha = q_\alpha(R|\pi_\theta)$, we obtain
$$ w(C,\tau) = \frac{P^{\pi_\theta}_{\phi_0}(C,\tau)}{P^{\pi_\theta}_{\phi_0,\alpha}(C,\tau)} = \frac{P^{\pi_\theta}_{\phi_0}(C,\tau)}{\alpha^{-1} \pmb{1}_{R(\tau)\le q_\alpha(R|\pi_\theta)} P^{\pi_\theta}_{\phi_0}(C,\tau)} = \alpha .$$

The assumption $\hat{q}_\alpha = q_\alpha(R|\pi_\theta)$, when applied to Equation~\eqref{eq:GCVaR_grad}, also guarantees that $\nabla_\theta \hat{J}_\alpha$ is an unbiased gradient estimator for both sample distributions $P=P^{\pi_\theta}_{\phi_0}$ and $P=P^{\pi_\theta}_{\phi_0,\alpha}$:
$\mathbb E_{C_i,\tau_{i}\sim P}[\nabla_\theta \hat{J}_\alpha \left( \{C_i,\tau_i\}_{i=1}^N;\, \pi_\theta \right)] = \nabla_\theta J_\alpha \left( \pi_\theta \right)$.
Its variance over $N$ i.i.d samples is
$\text{Var}_{C_i,\tau_{i}\sim P}[\nabla_\theta \hat{J}_\alpha \left( \{C_i,\tau_i\}_{i=1}^N;\, \pi_\theta \right)] = \frac{1}{N}\text{Var}_{C,\tau\sim P}[\nabla_\theta \hat{J}_\alpha \left( C,\tau;\, \pi_\theta \right)]$.
Denoting $g\coloneqq \nabla_\theta J_\alpha \left( \pi_\theta \right)$, we obtain:
\[
\begin{split}
&\text{Var}_{C,\tau \sim P^{\pi_\theta}_{\phi_0,\alpha}}[\nabla_\theta \hat{J}_\alpha \left( C,\tau;\,\pi_\theta \right)]\\ =& \mathbb E_{C,\tau \sim P^{\pi_\theta}_{\phi_0,\alpha}}[w(C,\tau)^2 H(C,\tau)^2] - g^2 \\
=& \mathbb E_{C,\tau \sim P^{\pi_\theta}_{\phi_0}} [ w(C,\tau)H(C,\tau)^2 ] - g^2 \\
= & \alpha \cdot \mathbb E_{C,\tau \sim P^{\pi_\theta}_{\phi_0}} [ H(C,\tau)^2 ] - g^2 \\
\le & \alpha \cdot (\mathbb E_{C,\tau \sim P^{\pi_\theta}_{\phi_0}} [ H(C,\tau)^2 ] - g^2) \\
= & \alpha \cdot \text{Var}_{C,\tau \sim P^{\pi_\theta}_{\phi_0}}[\nabla_\theta \hat{J}_\alpha \left( C,\tau;\, \pi_\theta \right)],
\end{split}
\]
which completes the proof.
\end{proof}

Note that if $\hat{q}_\alpha \ne q_\alpha(R|\pi_\theta)$, the term $\pmb{1}_{R(\tau)\le \hat{q}_\alpha}$ in the denominator may vanish and the IS weight $w(\tau,C)$ may become unbounded. To overcome this issue when using our CE-sampler (described in Section~\ref{sec:method}), we constrain the family of distributions $\{ P_{\phi}^{\pi_\theta} \}_\phi$ such that the sample distribution $P_{\phi}^{\pi_\theta}$ always has the same support as the original distribution $P_{\phi_0}^{\pi_\theta}$ (even though this eliminates the possibility of an exact tail sampling $P_{\phi}^{\pi_\theta} = P_{\phi_0,\alpha}^{\pi_\theta}$). In addition, in the experiments of Section~\ref{sec:experiments} we clip the IS weights directly.


\section{Gradient Estimation Bias and CeSoR Convergence}
\label{sec:analysis_gradient}

The gradient estimator of Equation~\eqref{eq:GCVaR_grad} is biased due to the biasedness of the empirical quantile.
However, \citet{cvar_via_sampling} show that the gradient estimator is still consistent, and bound its bias by $\mathcal{O}(N^{-1/2})$.
Lemma~\ref{lemma:gradient_bias} below proves that a similar result holds for CeSoR -- despite the CEM and the risk scheduling.
Given Lemma~\ref{lemma:gradient_bias}, CeSoR's convergence is a direct application of Theorem 5 in \citet{cvar_via_sampling}, as stated below.
The soft-risk scheduling $\alpha'$ introduces additional transient bias to the CVaR gradient estimate when $\alpha^\prime>\alpha$, but this bias vanishes in the last steady-state $1-\rho$ steps when $\alpha'=\alpha$; hence, we can safely assume consistency of CeSoR's gradient estimate, and focus our asymptotic convergence analysis on the steady-state phase.

Formally, in terms of Section~\ref{sec:preliminaries}, assume that the update step includes a $\ell_p$ projection $\Gamma$ to a compact set with a smooth boundary: $\theta_{m+1} = \Gamma(\theta_m + \eta_m \nabla_\theta \hat{J}_\alpha)$; and that the learning rate $\eta_m$ satisfies $\sum_{m=0}^\infty\eta_m=\infty$, $\,\sum_{m=0}^\infty\eta_m^2<\infty$ and $\sum_{m=0}^\infty\eta_m\left| E \left[ \nabla_\theta \hat{J}_\alpha \right] - \nabla_\theta J_\alpha \right|<\infty$ w.p.~1.
In addition, denote by $\mathcal{K}$ the set of all asymptotically-stable equilibria of the ODE $\,\dot{\theta} = \Gamma(\nabla_\theta {J}_\alpha(R; \pi_\theta))$.

\begin{theorem}[Convergence of CeSoR]
\label{theorem:convergence}
Assume that for any $\phi$, the sample distribution $D_\phi$ of Algorithm \ref{algo:cesor} has the same support as the original distribution $D_{\phi_0}$.
Then, under the smoothness assumptions specified in Appendix~\ref{sec:gradient_estimation}, and the projection and learning rate assumptions specified above, the sequence of policy parameters $\{\theta_m\}$ generated by Algorithm \ref{algo:cesor} converges almost surely to $\mathcal{K}$. 
\end{theorem}

Theorem~\ref{theorem:convergence} relies on similar assumptions to \citet{cvar_via_sampling}, two of them are of particular interest in our context.
First, the rewards are assumed to be continuous.
Second, in the gradient estimator, the baseline is assumed to be a consistent estimator of the returns $\alpha$-quantile. Hence, while CeSoR is compatible with any CVaR-PG method, the current derivation of theoretical convergence guarantees only holds for PG methods with a consistent gradient estimate.


\subsection{Gradient Estimation Bias}
\label{sec:gradient_estimation}


The gradient estimator of the standard CVaR PG may be inconsistent and unboundedly-biased, unless the return baseline is a consistent estimator of the $\alpha$-quantile of the returns~\citep{cvar_via_sampling}. Thus, we rely on the empirical quantile baseline $\hat{q}_\alpha$ used in Equation~\eqref{eq:GCVaR_grad}, which is a consistent (though biased) estimator of the true quantile.
Given certain smoothness assumptions, \citet{cvar_via_sampling} bound the resulted bias of the gradient estimator $E \left[ \nabla_\theta \hat{J}_\alpha \right] - \nabla_\theta J_\alpha$ (as defined in Equations~\eqref{eq:cvar},\eqref{eq:GCVaR_grad}).
Lemma~\ref{lemma:gradient_bias} guarantees that under the same assumptions, despite the modified sampling by the CEM, the same bias bounds apply to CeSoR.

We first specify the smoothness assumptions.
Note that \citet{cvar_via_sampling} consider $\nabla_\theta \log f_{s|a}(s|a,\theta)$ in their calculations (or in their notation: $\nabla_\theta \log f_{X|Y}(X|Y,\theta)$).
In RL applications, given the action $a$, the next-state distribution is independent of the policy $\pi_\theta$, and this gradient vanishes. We accordingly ignore this term in the calculations, which simplifies the assumptions and the analysis.
The remaining assumptions mostly consider the smoothness of the rewards, and in particular do not hold in the case of discrete rewards as discussed in Section~\ref{sec:blindness_to_success}.

\begin{assumption}[Smoothness assumptions]
\label{ass:smooth}
For any policy $\pi_\theta$, the return $R$ is a continuous random variable; and $\nabla_\theta q_\alpha(R;\pi_\theta)$, $\nabla_\theta J_\alpha(\pi_\theta)$ and $\nabla_\theta\log \pi_\theta(a)$ (for any $a$) are well defined and bounded.
\end{assumption}


\begin{lemma}[Gradient estimation bias bound]
\label{lemma:gradient_bias}
In Algorithm~\ref{algo:cesor} with a batch size $N$, consider a certain step $m\ge \rho M$, and assume that the underlying PG follows Equation~\eqref{eq:GCVaR_grad} (or Equation~\eqref{eq:GCVaR_grad_is}).
In addition, assume that for any $\phi$, the sample distribution $D_\phi$ of Algorithm \ref{algo:cesor} has the same support as the original distribution $D_{\phi_0}$.
Then, under Assumption~\ref{ass:smooth},
$E \left[ \nabla_\theta \hat{J}_\alpha \right] - \nabla_\theta J_\alpha = \mathcal{O}(N^{-1/2})$.
\end{lemma}

\begin{proof}

We follow the steps of the proof of Theorem 4 in \citet{cvar_via_sampling} with the following modifications.
First, we take the gradient expectations with respect to the CE sampling distribution $D_{\phi}$ rather than the original distribution $D_{\phi_0}$.
Second, the empirical quantile $\hat{q}_\alpha$ is calculated in Algorithm~\ref{algo:cesor} using a reduced sample size $N_o = \lfloor \nu N \rfloor < N$. Note that the estimator $\hat{q}_\alpha$ relies on samples drawn from $D_{\phi_0}$, hence is not otherwise affected by the CEM.


Denote by $D_{\phi_i}$ the distribution from which was drawn $C_i$, i.e., $\phi_i={\phi_0}$ for $i\le N_o$ and $\phi_i=\phi$ for $i>N_o$.
Since $m\ge\nu N$, according to Line~\ref{line:alpha_prime} in Algorithm~\ref{algo:cesor} we have $\alpha^\prime=\alpha$.
Denoting by $q_\alpha$ the true $\alpha$-quantile of the returns, we have
\begin{align}
    \nabla_\theta {J}_\alpha(R; \pi_\theta) = E_{\{\phi_i\}_{i=1}^N} \left[ \frac{1}{\alpha N} \sum_{i=1}^N w_i \pmb{1}_{R_i\le {q}_\alpha} \left(R_i-{q}_\alpha\right) \nabla_\theta\log \pi_\theta(\tau_i) \right]
\end{align}

We now substitute $w_i=\frac{D_{\phi_0}(C_i)}{D_{\phi_i}(C_i)}$, which is finite due to the assumption that $D_{\phi_i}$ has the same support as $D_{\phi_0}$.
Using the notation $E_{\phi_0}\left[ \cdot \right] = E_{C\sim D_{\phi_0},\, \tau\sim P_C^{\pi_\theta}}\left[ \cdot \right]$ and $\pi_\theta(\tau_i) = \Pi_t \pi_\theta(a_{i,t};s_{i,t})$, we obtain

\begin{align}
\begin{split}
    &\left| E_{C_i\sim D_{\phi_i},\,\tau_i\sim P_{C_i}^{\pi_\theta}} \left[ \nabla_\theta \hat{J}_\alpha(\{\tau_i\}; \pi_\theta) \right] - \nabla_\theta J_\alpha(\pi_\theta) \right| \\
    &\le E_{C_i\sim D_{\phi_i},\,\tau_i\sim P_{C_i}^{\pi_\theta}} \left[ \frac{1}{\alpha N} \sum_{i=1}^N \frac{D_{\phi_0}(C_i)}{D_{\phi_i}(C_i)} \left| \nabla_\theta\log \pi_\theta(\tau_i) \left( \pmb{1}_{R_i\le \hat{q}_\alpha} \left(R_i-\hat{q}_\alpha\right) - \pmb{1}_{R_i\le {q}_\alpha} \left(R_i-{q}_\alpha\right) \right) \right| \right] \\
    &= E_{\phi_0} \left[ \frac{1}{\alpha N} \sum_{i=1}^N \left| \nabla_\theta\log \pi_\theta(\tau_i) \left( \pmb{1}_{R_i\le \hat{q}_\alpha} \left(R_i-\hat{q}_\alpha\right) - \pmb{1}_{R_i\le {q}_\alpha} \left(R_i-{q}_\alpha\right) \right) \right| \right] \\
    &= E_{\phi_0} \left[ \frac{1}{\alpha N} \sum_{i=1}^N \left| \nabla_\theta\log \pi_\theta(\tau_i) \left( (\pmb{1}_{R_i\le \hat{q}_\alpha}-\pmb{1}_{R_i\le {q}_\alpha}) \left(R_i-\hat{q}_\alpha\right) + \pmb{1}_{R_i\le {q}_\alpha}(\left(R_i-\hat{q}_\alpha\right)- \left(R_i-{q}_\alpha\right) \right) \right| \right] \\
    &\le E_{\phi_0} \left[ \frac{1}{\alpha N} \sum_{i=1}^N \left| \nabla_\theta\log \pi_\theta(\tau_i) (\pmb{1}_{R_i\le \hat{q}_\alpha}-\pmb{1}_{R_i\le {q}_\alpha}) \left(R_i-\hat{q}_\alpha\right) \right| \right] \\
    &\phantom{\le\le} + E_{\phi_0} \left[ \frac{1}{\alpha N} \sum_{i=1}^N \left| \nabla_\theta\log \pi_\theta(\tau_i) \pmb{1}_{R_i\le {q}_\alpha}\left({q}_\alpha-\hat{q}_\alpha\right) \right| \right] \\
\end{split}
\end{align}

From this point, the proof is mostly identical to Theorem 4 in \citet{cvar_via_sampling}.
Namely, the first term is $o(N^{-1/2})$ according to \citet{cvar_sensitivity}, given Assumption~\ref{ass:smooth}; and since $\hat{q}_\alpha$ is estimated using $\nu N$ samples, we have $\left|{q}_\alpha-\hat{q}_\alpha\right| = \mathcal{O}((\nu N)^{-1/2}) = \mathcal{O}(N^{-1/2})$ in probability (note that $\nu$ is constant, e.g., $\nu=0.2$ or $\nu=0.5$ in the experiments of Section~\ref{sec:experiments}).
Together, the whole expression is $\mathcal{O}(N^{-1/2})$ as required.

\end{proof}





\section{The Cross Entropy Module: Extended Discussion}
\label{sec:CE_results}

The Cross Entropy Method (CEM) with non-stationary score function has a major role in CeSoR.
The CEM \code{CEM.py}{code} is implemented and available as an independent module~\citep{cem_nonstationary}.
Below we present an analysis of the CEM empirical results over both a dedicated toy problem (which tests the CEM independently of CeSoR) and as part of CeSoR in the benchmarks of Section~\ref{sec:experiments}.

\subsection{The CEM Algorithm}
\label{sec:CE_algo}

For clarity, we first provide the pseudo-code for the general CEM algorithm.
This version repeatedly generates samples from the tail of the distribution $D_{\phi_0}$. A similar version~\citep{CE_tutorial} would stop once $q_\beta\left(\{R(x_i)\}_{i=1}^N\right) \le q$ (as it means that at least $\beta N$ samples are already beyond $q$), and use all the recent samples $R(x_i)\le q$ to estimate the probability of the ``rare event'' $R(X)\le q$.

Note that unlike CeSoR, Algorithm~\ref{algo:CEM} relies on a constant mapping $R(x)$ and a constant target $q$.
Our CEM version in CeSoR, as implemented in our code and presented in Algorithm~\ref{algo:cesor}, supports a quantile-target $\alpha$ with respect to a return mapping $R$ that varies dynamically with the learning agent.

\begin{algorithm}[H]
\caption{The Cross Entropy Method for Sampling}
\label{algo:CEM}
\DontPrintSemicolon
\SetAlgoNoLine
\SetNoFillComment

 {\bf Input}: distribution $D_{\phi_0}$; score function $R$; target level $q$; batch size $N$; update selection rate $\beta$.\;
 \BlankLine
 \BlankLine
 $\phi\leftarrow {\phi_0}$\;
 
 \While{true}{
    \tcp{Sample}
    Sample $x \sim D_{\phi}^N$\;
	$w_i \leftarrow D_{{\phi_0}}(x_i) / D_{\phi}(x_i) \quad (1\le i\le N)$\;
	Print $x$\;
    \tcp{Update}
    $q^\prime \leftarrow \max\left(q,\ q_\beta\left(\{R(x_i)\}_{i=1}^N\right) \right)$\;
	$\phi \leftarrow \mathrm{argmax}_{\phi^{\prime}} \sum_{i=1}^N w_i \pmb{1}_{R(x_i)\le q^\prime} \log D_{\phi^{\prime}}(x_i)$\;
 }
\end{algorithm}

\subsection{Sample Distribution}
\label{sec:ce_sample_dist}
The goal of the CEM is to align the sample distribution with the bottom-$\alpha$ percent of the reference distribution.
Note that given a parametric family of distributions $D_\phi$ with a limited expressiveness, a perfect alignment is not always possible. For example, if the CEM controls the mean of an exponential distribution $C\sim Exp(\phi)$, and the returns decrease with $c$, then the lower quantiles of the returns correspond to $C\ge q_\alpha(C)$. However, no value of $\phi$ could eliminate the lower values $C\in\left[0, q_\alpha\right]$ -- but could merely assign more probability density to higher values.
Even when the family of distributions is expressive enough, the CEM has to learn the desired sample distribution without any prior knowledge about the meaning of the parameters that it controls. In particular, it cannot know in advance in which direction each parameter may affect the agent return, what the size of the effect would be, and how it would change during the training.

Formally, the objective of the CEM is often defined as minimization of the KL-divergence between the sample distribution and the desired tail of the reference distribution~\citep{ce_convergence_issues}.
Indeed, this objective is well-defined even if the expressiveness of $D_\phi$ does not allow a perfect alignment.

In this section, we focus on the comparison between the mean and the CVaR of the sample distribution and the reference distribution of the returns.
Specifically, while both distributions begin with the same mean and CVaR, we hope that the sample mean would align with the reference CVaR as quickly as possible.

First, we consider a toy problem with a static reference distribution and no RL environment.
The parametric family of distributions is $C\sim Beta(2\phi, 2-2\phi)$ (such that $E\left[C\right]=\phi$), and the reference distribution corresponds to ${\phi_0}=0.5$, which results in the uniform distribution $Beta(1,1)=U(0,1)$. We are interested in the bottom $\alpha=10\%$ of the reference distribution, i.e., $U(0,0.1)$.
We run the CEM for $n=10$ steps with $N=1000$ samples per step, $\nu=20\%$ of them are drawn from the original reference distribution, and update $\phi$ using the mean of the lower $\beta=50\%$ samples. Note that generally in this work, $C$ is the context or configuration of an environment that produces returns; in this toy example, we do not have an RL environment and we simply define $R(C)=C$.

\begin{figure}[!ht]
\centering
  \includegraphics[width=.4\linewidth]{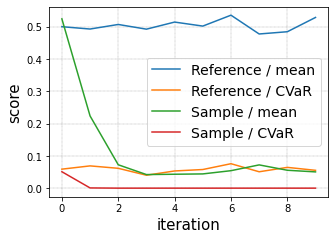}
\caption{\small The converges of the CE sample mean to the reference $CVaR_{10\%}$ in the toy $Beta$ distribution problem.}
\label{fig:ce_demo}
\end{figure}

The $CVaR_{10\%}$ of $C$ (or equivalently, the mean of $U(0,0.1)$) is $0.05$. Note that no value of $\phi$ can yield the distribution $U(0,0.1)$, as the support of the $Beta$ distribution is always $(0,1)$.
Yet, as shown in Figure~\ref{fig:ce_demo}, the sample mean converges to the reference CVaR within mere 2 iterations, and remains around this level.

Figure~\ref{fig:ce_scores} presents the same metrics for the experiments described in Section~\ref{sec:experiments}. In these cases, the reference returns distribution corresponds to the agent returns under the original environment. Note that this reference returns distribution is dynamic during the training, as it changes with the agent (and in certain benchmarks also with the episode length that increases throughout the training).
Yet, in the Driving Game benchmark, for example, we see that the sample mean reasonably aligns with the reference CVaR throughout most of the training, even as both of them vary.

\begin{figure}[!ht]
\centering
\begin{subfigure}{.33\textwidth}
  \centering
  \includegraphics[width=1.\linewidth]{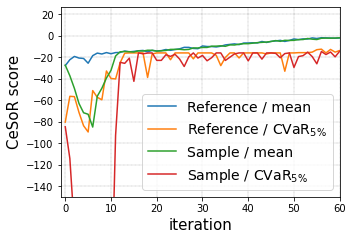}
  \caption{Guarded Maze (first 60 iterations)}
  \label{fig:ce_maze}
\end{subfigure}
\begin{subfigure}{.33\textwidth}
  \centering
  \includegraphics[width=1.\linewidth]{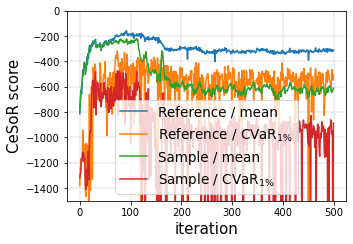}
  \caption{Driving Game}
  \label{fig:ce_driving}
\end{subfigure}
\begin{subfigure}{.32\textwidth}
  \centering
  \includegraphics[width=1.\linewidth]{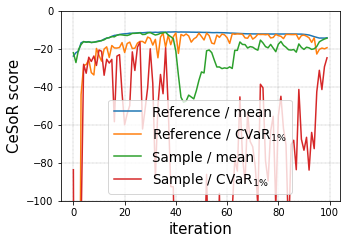}
  \caption{Servers Allocation}
  \label{fig:ce_servers}
\end{subfigure}
\caption{\small The mean and CVaR metrics of the CE sample distribution and the reference original distribution throughout the training of CeSoR over different benchmarks.}
\label{fig:ce_scores}
\end{figure}

In the Guarded Maze, the sample mean also quickly converges into the reference CVaR. However, once the agent learns to avoid the short path, the CE sampler can no longer control the agent performance at all, and due to the regularizing reference samples, the sample distribution gradually goes back to the original one. This is a valid behavior, as the agent already learned to avoid the risk, and if for some reason it came back to the risky short path, the CE would simply learn again to focus on the risky configurations of the environment.

The Servers Allocation Problem takes the challenge of the CEM to the limit, as the target is $\alpha=1\%$, the difficulty to the agent arrives in a non-smooth manner as rare and discrete events, and the given family of distributions (Binomial) has limitations in expressing the desired distribution.
Specifically, we would like most of the sample episodes to include a peak event, but not more than one; whereas the Binomial distribution is not best-suitable for this.
However, even as the CEM struggles to fit the reference $CVaR_{1\%}$ (Figure~\ref{fig:ce_servers}), CeSoR is still shown to provide beneficial results (Section~\ref{sec:servers}, Appendix~\ref{sec:servers_detailed}).
This demonstrates the robustness of CeSoR to limitations and misspecification of the modeled family of distributions.

\paragraph{Sensitivity to $\beta$:}
As discussed in Sections \ref{sec:preliminaries} and \ref{sec:method}, the smoothness parameter $\beta$ determines the minimal percent of data samples used for the update step in the CEM.
We argue that CeSoR has a low sensitivity to the parameter $\beta$.

Intuitively, every iteration of the CEM focuses on the $\beta$-tail of the previous iteration (until reaching the $\alpha$-tail of the reference distribution).
Theoretical analysis of the convergence rate is challenging, due to the limited expressiveness of $D_\phi$ and the non-stationary agent returns; yet, according to the qualitative intuition above, we expect exponential convergence to the tail, which applies even for high values of $\beta$.
On the other hand, while low values of $\beta$ may increase the noise in the update step of the CEM, any noisy update could be corrected throughout the training. Note that Algorithm~\ref{algo:cesor} uses the original context-distribution for a certain part of the samples of each batch; this guarantees that any update step is reversible, as CeSoR continues to be exposed to the complete context-space.

Empirically, we repeated the experiments of Section~\ref{sec:experiments} with various values of $\beta\in[0.05,0.3]$. In the Guarded Maze and the Driving Game, all the values of $\beta$ resulted in similar test returns; in addition, Figure~\ref{fig:beta} shows that the CEM successfully aligned the sample mean return with the reference CVaR, independently of $\beta$.
The Servers Allocation Problem is more challenging for the CEM (as discussed above), making the sampler more sensitive to the parameter $\beta$, and in particular leading to a failure for $\beta=0.3$. However, note that even under such a combination of poor algorithmic choices (Binomial parameterization of $D_\phi$ and very high $\beta$), the failure of the CEM is easy to notice through Figure~\ref{fig:beta_servers} (as the sample-mean fails to deviate from the reference-mean), and is easy to fix.


\begin{figure}[!ht]
\centering
\begin{subfigure}{1.\textwidth}
    \centering
    \includegraphics[width=1.\linewidth]{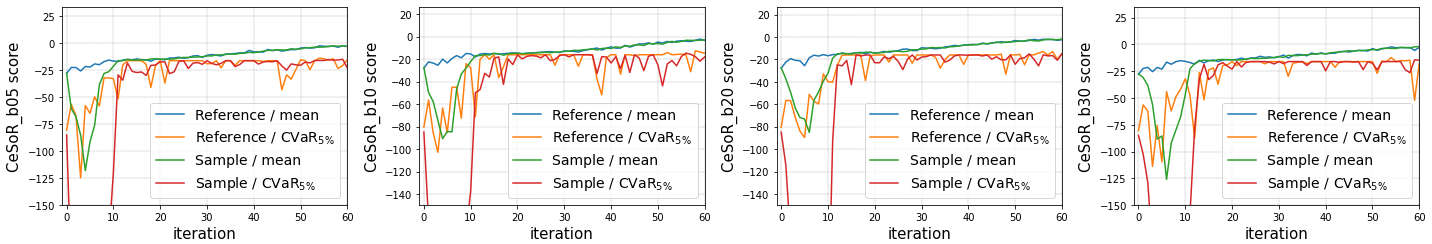}
    \caption{Guarded Maze}
\end{subfigure}
\begin{subfigure}{1.\textwidth}
    \centering
    \includegraphics[width=1.\linewidth]{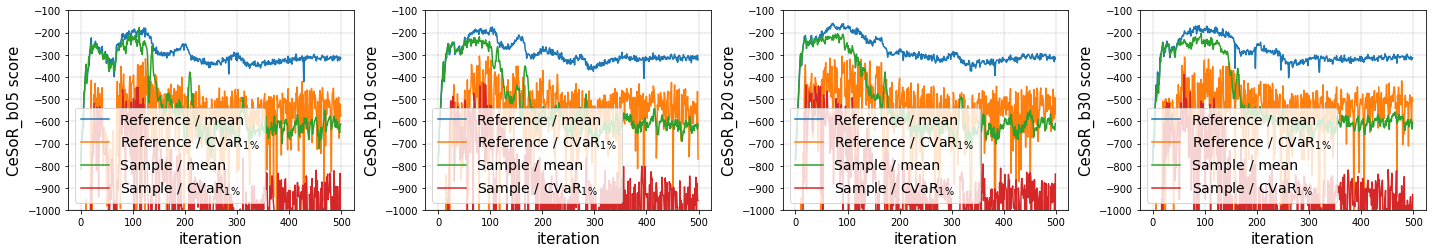}
    \caption{Driving Game}
\end{subfigure}
\begin{subfigure}{1.\textwidth}
    \centering
    \includegraphics[width=1.\linewidth]{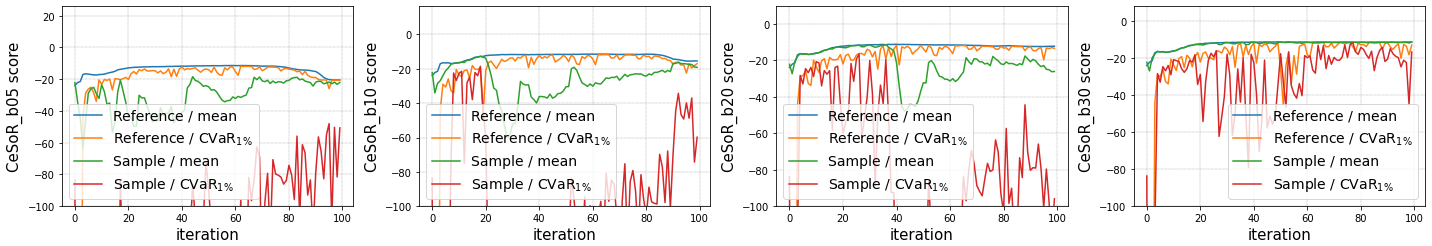}
    \caption{Servers Allocation}
    \label{fig:beta_servers}
\end{subfigure}
\caption{\small Returns statistics of the sample distribution and the reference original distribution, throughout the training of CeSoR over different benchmarks, for different values of $\beta\in[0.05,0.3]$.}
\label{fig:beta}
\end{figure}

\FloatBarrier


\subsection{Sample Efficiency}
\label{sec:ce_sample_efficiency}

An important aspect of the CEM is its increase of sample efficiency (Section~\ref{sec:variance_reduction}).
While the results in Section~\ref{sec:experiments} already demonstrate that CeSoR learns better and faster than the standard GCVaR, here we measure the effective sample size directly.
While PG always uses the entire batch, and GCVaR always uses at most $\alpha$ of the episodes, Figure~\ref{fig:sample_size} shows that CeSoR manages to optimize CVaR$_\alpha$ while using more than $\alpha$ percent of the data. Note that even beyond the risk level scheduling (which ends after $\rho=80\%$ of the training), the CEM still allows for more than $\alpha$ percent of each batch to be used.

\begin{figure}[!ht]
\centering
\begin{subfigure}{.8\textwidth}
  \centering
  \includegraphics[width=1.\linewidth]{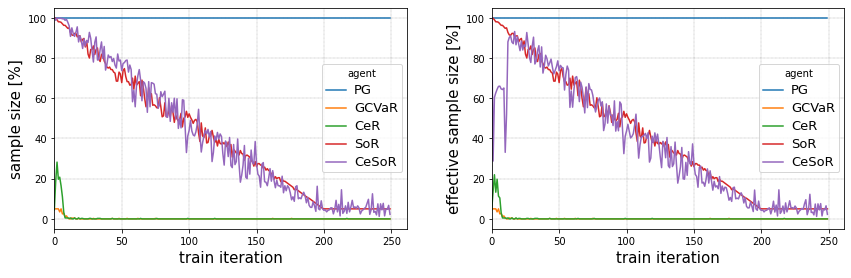}
  \caption{Guarded Maze}
  \label{fig:maze_sample_size}
\end{subfigure} \\
\begin{subfigure}{.8\textwidth}
  \centering
  \includegraphics[width=1.\linewidth]{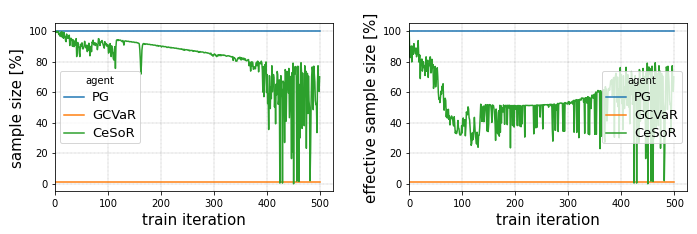}
  \caption{Driving Game}
  \label{fig:driving_sample_size}
\end{subfigure} \\
\begin{subfigure}{.8\textwidth}
  \centering
  \includegraphics[width=1.\linewidth]{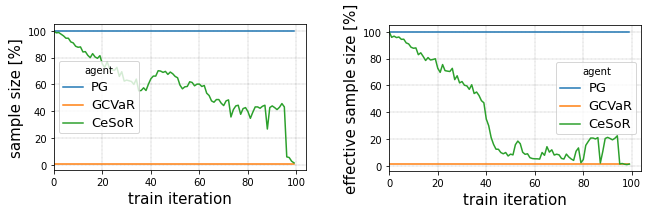}
  \caption{Servers Allocation}
  \label{fig:servers_sample_size}
\end{subfigure}
\caption{\small Left -- sample size: the percent of episode samples (out of $N=400$ episodes per training iteration) used by the optimizer. Note that only returns $R(\tau_i)<q_\alpha$ are counted (strict inequality), since the contribution of episodes with $R(\tau_i)=q_\alpha$ to the loss is $0$ (Equation~\eqref{eq:GCVaR_grad}). Right -- \textit{effective} sample size: this takes into account the IS weights: the effective sample size equals the number of equally-weighted independent samples needed to obtain the same estimation variance~\citep{n_eff_book,n_eff_blog}: $n_{eff}=(\sum_iw_i)^2/\sum_iw_i^2$. Note that for equal weights, $n_{eff}=n$.}
\label{fig:sample_size}
\end{figure}

Note that GCVaR effectively uses \textit{less} than $\alpha$ episodes in a batch if multiple episodes $\{\tau_i\}$ satisfy $R(\tau_i)=q_\alpha$ -- since the contribution of any such episode to the gradient in Equation~\eqref{eq:GCVaR_grad} is $0$.
In the extreme case, as discussed in in Section~\ref{sec:analysis_blindness} and Appendix~\ref{sec:maze_detailed_results}, all the worst $\alpha$ episodes are identical, and the whole loss gradient is identically $0$.

\FloatBarrier


\subsection{Risk Characterization}

The CEM not only allows CeSoR to sample the most relevant environment conditions for CVaR optimization, but also allows us to characterize the conditions that correspond to the risk level $\alpha$. This enhances our understanding of the problem and may help us to anticipate poor returns in advance.

\begin{figure}[!ht]
\centering
\begin{subfigure}{.67\textwidth}
  \centering
  \includegraphics[width=1.\linewidth]{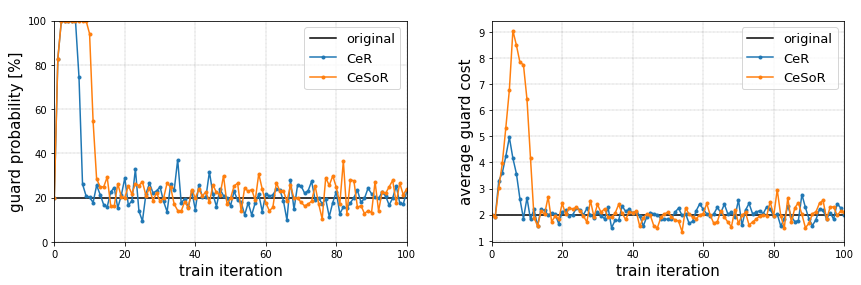}
  \caption{Guarded Maze}
  \label{fig:ce_maze_dist}
\end{subfigure}
\begin{subfigure}{.32\textwidth}
  \centering
  \includegraphics[width=1.\linewidth]{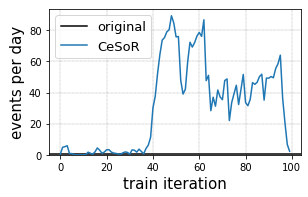}
  \caption{Servers Allocation}
  \label{fig:ce_servers_dist}
\end{subfigure}
\begin{subfigure}{.32\textwidth}
  \centering
  \includegraphics[width=1.\linewidth]{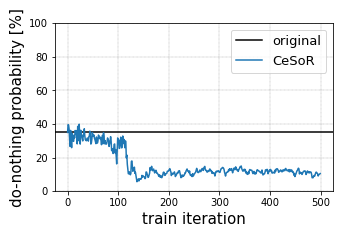}
  \caption{Driving}
  \label{fig:ce_driving_p_nothing}
\end{subfigure}
\begin{subfigure}{.32\textwidth}
  \centering
  \includegraphics[width=1.\linewidth]{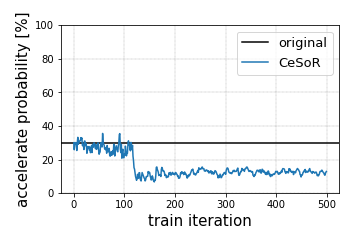}
  \caption{Driving}
  \label{fig:ce_driving_p_acc}
\end{subfigure}
\begin{subfigure}{.32\textwidth}
  \centering
  \includegraphics[width=1.\linewidth]{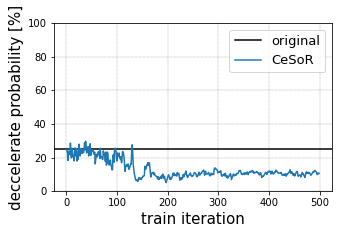}
  \caption{Driving}
  \label{fig:ce_driving_p_dec}
\end{subfigure}
\begin{subfigure}{.32\textwidth}
  \centering
  \includegraphics[width=1.\linewidth]{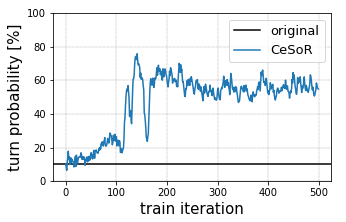}
  \caption{Driving}
  \label{fig:ce_driving_p_turns}
\end{subfigure}
\begin{subfigure}{.32\textwidth}
  \centering
  \includegraphics[width=1.\linewidth]{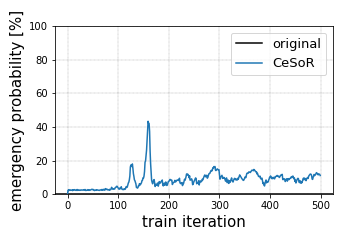}
  \caption{Driving}
  \label{fig:ce_driving_p_emergency}
\end{subfigure}
\caption{\small The evolution of the CE distribution parameters $\phi^\prime$ throughout the training in various benchmarks.}
\label{fig:ce_dist}
\end{figure}

Figure~\ref{fig:ce_dist} presents the evolution of the sample distribution parameters $\phi$ throughout the CeSoR training process in the various benchmarks.
In the Guarded Maze, for example, $\phi$ goes back to its original values once the agent behavior converges, which teaches us that a risk-averse agent can be entirely insensitive to the environment conditions.
In the Driving Game, on the other hand, the agent must still beware a leader that applies many turns and emergency brakes. Furthermore, the CEM provides the connection between the risk level of interest ($\alpha$) and the corresponding values of $\phi$ (e.g., how many turns and brakes it takes to bring us to this risk level).

\FloatBarrier


\section{The Guarded Maze: Extended Discussion}
\label{sec:maze_detailed}

\subsection{Implementation Details}
\label{sec:maze_implementation}

In this section we specify the implementation details of the Guarded Maze.
The full code is available in the \code{Examples/GuardedMaze/GuardedMaze.py}{gym environment} and the corresponding \code{Examples/GuardedMaze/GuardedMazeExample.ipynb}{jupyter notebook}.

\paragraph{The Guarded Maze benchmark:}
The benchmark introduces a maze of size $8\times 8$, with the walls marked in gray in Figure~\ref{fig:maze_nice}.
The target is a $1\times1$ square marked in green.
Every episode, the initial agent location is drawn from a uniform distribution over the lower-left quarter of the maze.
Every time step, the agent can walk in one of the directions left, right, up and down, with a step size of $1$, and an additive normally-distributed noise with standard deviation of $0.2$ in each dimension. That is,
$$ s_{t+1} = s_t + a_t + (\epsilon_1,\epsilon_2)^\top $$
where $s_t,a_t\in\mathbb{R}^2$ and $\epsilon_i\sim \mathcal{N}(0,0.2^2)\ (i\in\{1,2\})$.
A step that ends in a wall is cancelled, and the agent remains in its place.

Every time-step, the agent observes its location $s_t$.
In practice, we use a soft (continuous) one-hot encoding of the agent location in the maze, calculated as a 2D interpolation between the 4 nearest points of a $8\times8$ grid, represented as a corresponding $8\times8$ matrix.
That is, if the agent is located between the grid points $(i,j),(i,j+1),(i+1,j),(i+1,j+1)$, then all the other elements of the matrix are set to $0$, and these 4 elements are assigned positive value that are summarized to $1$, according to the relative location of the agent between them.
Note that the locations of the target and the guarded zone are constant, and are not given as input.

An episode ends either when reaching the target or after $160$ time-steps.
The rewards are specified in Section~\ref{sec:maze}.
The return of an episode is the sum of its rewards (i.e., no discount factor).
The maze is designed such that the $mean$-optimal strategy is taking the shortest path to the target, where the expected cost of crossing the guarded zone is $E\left[C_1C_2\right]=\phi_1\phi_2=0.2\cdot 32=6.4$ -- smaller than the additional cost of the longer path.
The $CVaR_{0.05}$-optimal strategy, however, is to take the longer path, since sometimes short cuts make long delays~\citep{tolkien}.

\paragraph{Algorithms implementation:}
The training algorithms are specified in Section~\ref{sec:experiments}.
In the maze benchmark, all of them are applied to a linear model that takes as an input the one-hot encoding described above ($\in\mathbb{R}^{64}$), and is followed by a softmax operator with temperature $T$.
That is, $P(a_j;\theta) = exp(Ty_j) / \sum_{j'} exp(Ty_{j'})$ (where $1\le j\le4$ and $y_j$ is the corresponding output of the linear model $F_\theta$).
We set a constant $T=1$ over the whole training, and $T=0$ (i.e., choosing the max-probability action) for validation and test episodes.


The CE module in CeR and CeSoR controls the parameters $\phi$ of the Bernoulli and the Exponential distributions.
Note that the module is aware of the original ("true") values of $\phi$, but not of their semantic meaning in the maze (e.g., it is not aware that high values are "bad", or that they only affect the agent through the guarded zone).
The sample parameters update using the moments-method is as simple as $\phi \leftarrow (mean(C_1), mean(C_2))$, calculated over the episodes selected by the CE (Line~\ref{line:update} in Algorithm~\ref{algo:cesor}).



\subsection{Detailed Results}
\label{sec:maze_detailed_results}

Figure~\ref{fig:maze_test_scores_full} shows the distribution of the trained agent returns over the test episodes in the Guarded Maze (note that the left tail of this distribution is displayed in Figure~\ref{fig:maze_test_scores}.
Figure~\ref{fig:maze_train} shows the mean and CVaR of the training and validation scores throughout the training process.
Below we elaborate on the training dynamics in general, and the blindness to success in particular.

\begin{figure}[!ht]
\centering
  \includegraphics[width=.4\linewidth]{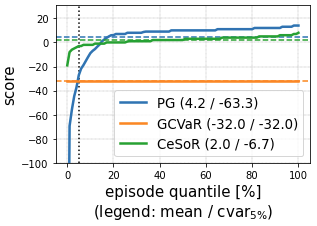}
\caption{\small The full distribution of the trained agent returns over the test episodes in the Guarded Maze. Note that Figure~\ref{fig:maze_test_scores} displays the left tail of the same distribution.}
\label{fig:maze_test_scores_full}
\end{figure}

\begin{figure}[!ht]
\centering
  \includegraphics[width=1.\linewidth]{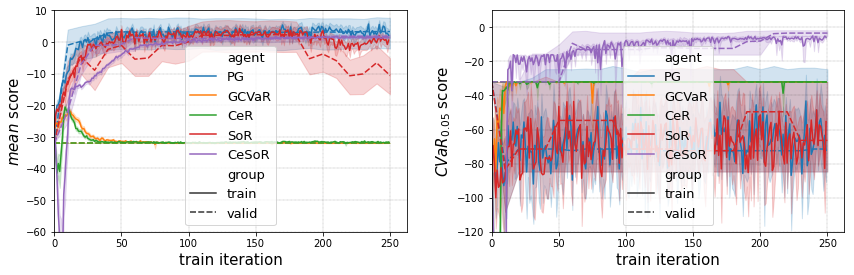}
\caption{\small Mean and CVaR scores over the train and validation episodes throughout the Guarded Maze training. The shading corresponds to 95\% confidence-intervals, based on bootstrapping over the episode-samples. Note that validation and train policies are not entirely identical, as the former deterministically chooses the action of max-probability (temperature $T=0$), and the latter operates stochastically ($T=1$).}
\label{fig:maze_train}
\end{figure}

\FloatBarrier

\paragraph{Blindness to success:}
Section~\ref{sec:ce_sample_efficiency} discusses the contribution of the \textit{CE sampling} to the sample efficiency. Here we discuss the contribution of \textit{soft risk level scheduling} to the sample efficiency, and in particular its prevention of \textit{blindness to success}.

\begin{figure}[!ht]
\centering
\begin{subfigure}{.33\textwidth}
  \centering
  \includegraphics[width=1.\linewidth]{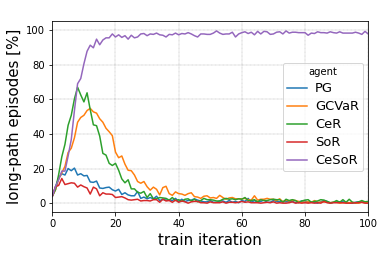}
  \caption{}
  \label{fig:maze_longs_tot}
\end{subfigure}
\begin{subfigure}{.35\textwidth}
  \centering
  \includegraphics[width=1.\linewidth]{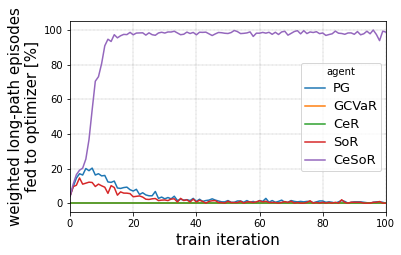}
  \caption{}
  \label{fig:maze_longs_fed}
\end{subfigure}
\begin{subfigure}{.3\textwidth}
  \centering
  \includegraphics[width=1.\linewidth]{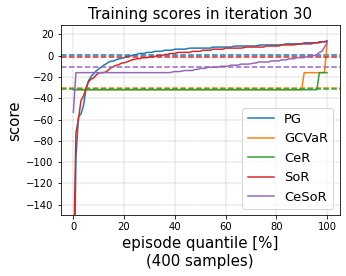}
  \caption{}
  \label{fig:maze_scores_30}
\end{subfigure}
\begin{subfigure}{1\textwidth}
  \centering
  \includegraphics[width=1.\linewidth]{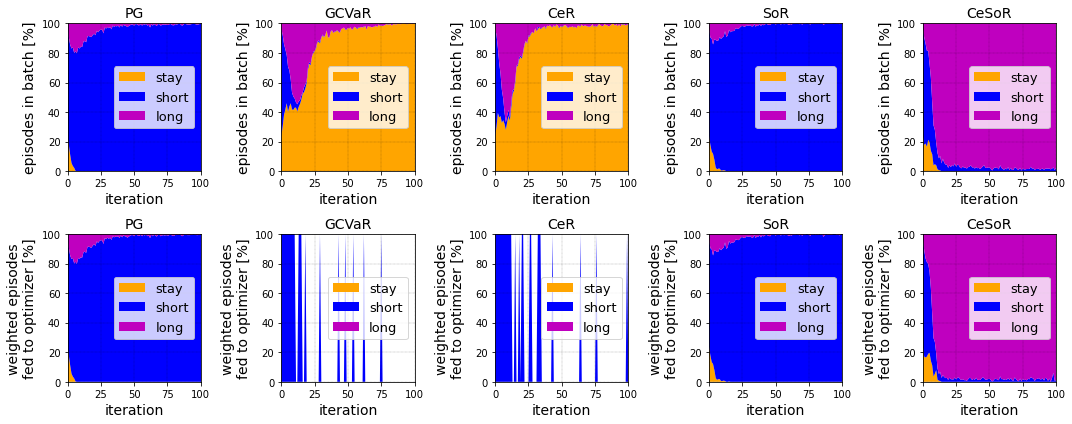}
  \caption{}
  \label{fig:maze_exposure_full}
\end{subfigure}
\caption{\small For the first 100 iterations of the Guarded Maze training, (a) the percent of episodes that reached the target through the long path; (b) the total weight of such long-path episodes that were fed to the optimizer (out of the total weight of episodes fed to the optimizer); (c) the returns distribution over the 30th training batch; and (d) percent of episodes (top) and total weight (bottom) for all 3 agent strategies (not only long path as in (a),(b)).}
\label{fig:blindness_to_success}
\end{figure}

As displayed in Figure~\ref{fig:maze_longs_tot}, for all the agents in the beginning of the optimization process, around 10\% of the episodes in every batch reach the target through the long path. At the same time, around 70\% of the episodes reach the target through the short (and risky) path.
As a risk-averse algorithm, GCVaR learns to avoid the short path, and the ratio of the long-path episodes increases accordingly -- reaching up to 50\% around the 15th batch (recall that in training episodes the actions are selected randomly according to the policy softmax output with temperature 1, which allows the agent to randomly reach the target).
Nonetheless, as shown in Figure~\ref{fig:maze_longs_fed}, in \textit{all} of the train iterations, \textit{none} of the long-path episodes belong to the bottom $\alpha=5\%$ episodes (which are fed to the optimizer), hence GCVaR never learns to prefer the long-path.
This demonstrates the blindness of GCVaR to the successful long path.


In fact, after around 10 training iterations of GCVaR, all the bottom $\alpha=5\%$ episodes in most batches already follow the stay-strategy (i.e., do not reach the target, nor take the guarded-zone risk), and achieve a constant return of $-32$ (Figure~\ref{fig:maze_scores_30}). Note that according to Equation~\eqref{eq:GCVaR_grad}, this means that the loss gradient is identically $0$.
As shown in Figure~\ref{fig:maze_sample_size}, the used sample size of GCVaR is indeed 0 after the 10th iteration, the effective sample efficiency is 0, and most of the changes in the agent from this point are attributed to the remaining Adam gradient momentum.

The soft risk level scheduling eliminates the blindness to success, and allows the optimizer to observe the long-path episodes (SoR in Figure~\ref{fig:maze_longs_fed}). However, at the same time, it reduces the risk-aversion of the agent, and the long path is no longer preferred over the short path. When the risk level reduces sufficiently, the agent may re-learn to avoid the short path, but the long path is no longer sampled at all and cannot be learned.

Only CeSoR manages both to observe the long-path episodes (thanks to soft risk level scheduling) \textit{and} to prefer them over the short path (thanks to the risk-aversion induced by the CEM).

\FloatBarrier

\paragraph{Examples and visualization:}
Figure~\ref{fig:maze_policies} visualizes the policies learned by PG, GCVaR and CeSoR.
While the policies are defined over all the continuous state space, the visualization is restricted to a discrete grid.
Note that CeSoR and GCVaR behave similarly in the lower-left part of the maze, corresponding to guarded-zone avoidance; however, since GCVaR never observed the long path and learned its benefits, it fails to learn the CVaR-optimal strategy in the upper part of the maze.

\begin{figure}[!ht]
\centering
\begin{subfigure}{.25\textwidth}
  \centering
  \includegraphics[width=1.\linewidth]{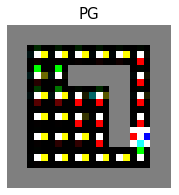}
  \caption{}
  \label{fig:maze_pg}
\end{subfigure}
\begin{subfigure}{.25\textwidth}
  \centering
  \includegraphics[width=1.\linewidth]{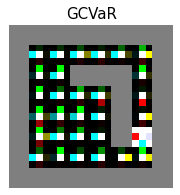}
  \caption{}
  \label{fig:maze_gcvar}
\end{subfigure}
\begin{subfigure}{.25\textwidth}
  \centering
  \includegraphics[width=1.\linewidth]{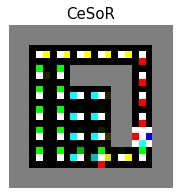}
  \caption{}
  \label{fig:maze_cesor}
\end{subfigure}
\caption{\small The policies learned by PG, GCVaR and CeSoR, visualized over a discrete grid within the continuous state space of the Guarded Maze. The colors brightness around each point in the grid corresponds to the probabilities assigned to the actions by the policy given this point.}
\label{fig:maze_policies}
\end{figure}

Figure~\ref{fig:maze_examples} shows a sample of test episodes for each of the trained agents.
Due to the reduced risk-aversion of SoR (as discussed above), its best validation CVaR score was obtained early in the training, which may explain its non-smooth behavior in Figure~\ref{fig:maze_examples}.

\begin{figure}[!ht]
\centering
  \includegraphics[width=.8\linewidth]{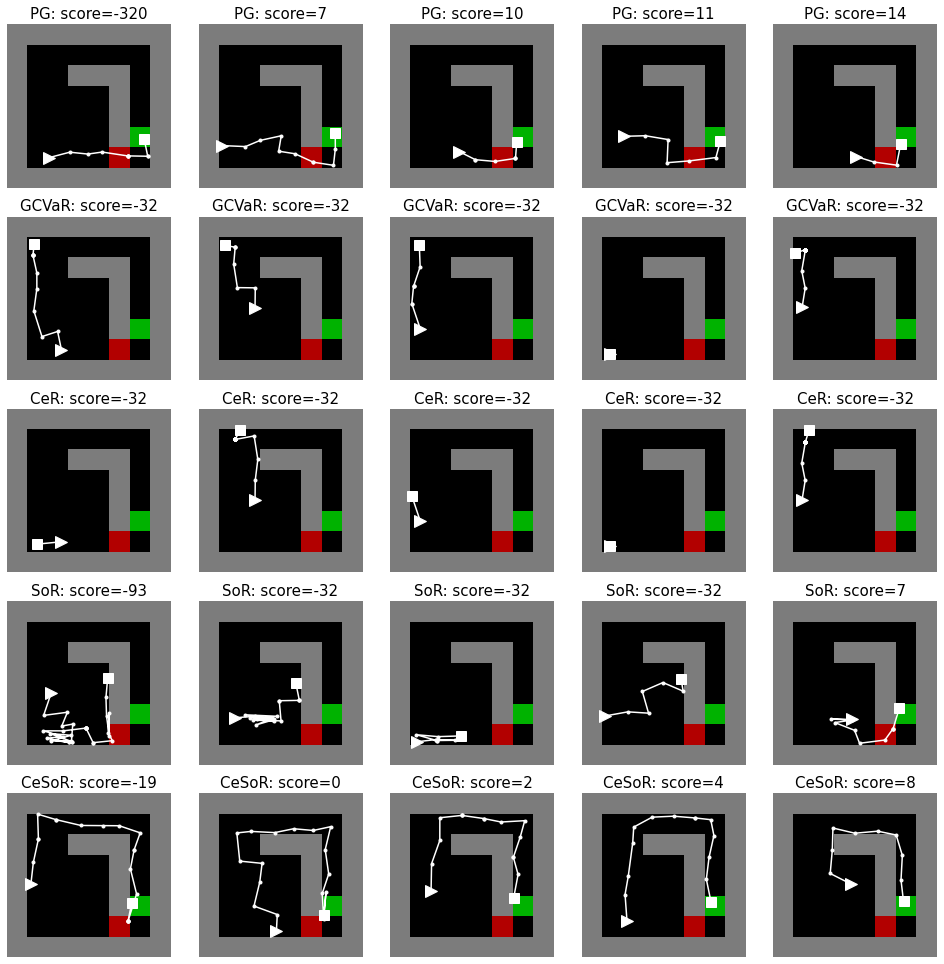}
\caption{\small A sample of test episodes for each of the trained agents in the Guarded Maze.}
\label{fig:maze_examples}
\end{figure}

\FloatBarrier


\section{The Driving Game: Extended Discussion}
\label{sec:driver_detailed}

\subsection{Implementation Details}
\label{sec:driver_implementation}

In this section we specify the implementation details of the Driving Game.
The full code is available in the \code{Examples/DrivingGame/DrivingSim.py}{gym environment} and the corresponding \code{Examples/DrivingGame/DrivingExample.ipynb}{jupyter notebook}.
Note that the leader behavior generation mechanism and the policy architecture are already specified in Section~\ref{sec:experiments}.

\textbf{Observation space}: the policy receives the following variables as inputs: relative position $dx,dy$, relative on-track velocity $dvx$, agent acceleration $ax$ and agent direction $\theta$.

\textbf{Action space}: the possible agent actions are (1) keep speed and steer; (2) accelerate; (3) decelerate; (4) steer left; (5) steer right.
The acceleration and deceleration magnitudes ($+4m/s^2, -6m/s^2$) were determined according to the typical acceleration value described in \citet{driving2018}.

\textbf{Rewards}: we use the rewards defined in \citet{driving2018}, with the parameters $r_1=0.5, r_2=0.05, r_3=0.1, r_4=0.5, r_5=1, r_6=0.5$.
These parameters determine the scale of the 6 additive rewards of \citet{driving2018}, which correspond to staying behind the leader, staying close to the leader, keeping similar speed to the leader, keeping smooth agent acceleration, staying in the same lane as the leader, and staying on-road, respectively.
We also add a new additive reward of size $5$ for any time-step with overlap between the agent and leader cars, meant to penalize collisions -- which are not explicitly expressed in the original rewards.


\subsection{Detailed Results}
\label{sec:driver_detailed_results}

Figures~\ref{fig:driving_test_scores_full}-\ref{fig:driving_frames} present a detailed analysis of the results of the Driving Game experiments.

\begin{figure}[!ht]
\centering
  \includegraphics[width=.4\linewidth]{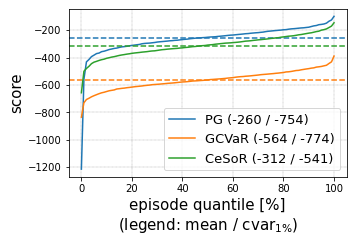}
\caption{\small The full distribution of the trained agent returns over the test episodes in the Driving Game. Note that Figure~\ref{fig:driving_test_scores} displays the left tail of the same distribution.}
\label{fig:driving_test_scores_full}
\end{figure}

\begin{figure}[!ht]
\centering
\begin{subfigure}{.4\textwidth}
  \centering
  \includegraphics[width=1.\linewidth]{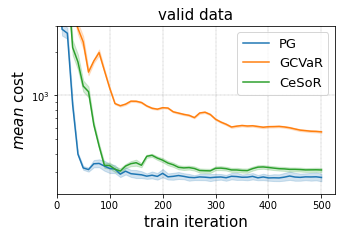}
  \caption{$Mean$}
  \label{fig:driving_valid_mean}
\end{subfigure}
\begin{subfigure}{.4\textwidth}
  \centering
  \includegraphics[width=1.\linewidth]{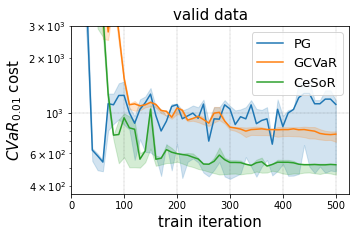}
  \caption{$CVaR_{1\%}$}
  \label{fig:driving_valid_cvar}
\end{subfigure}
\caption{\small Mean and CVaR scores over the validation episodes throughout the Driving Game training. The shading corresponds to 95\% confidence-intervals, based on bootstrapping over the episode-samples.}
\label{fig:driving_valid}
\end{figure}

\begin{figure}[!ht]
\centering
\begin{subfigure}{.25\textwidth}
  \centering
  \includegraphics[width=1.\linewidth]{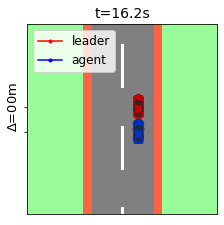}
  \caption{}
  \label{fig:driving_pg_frame}
\end{subfigure}
\begin{subfigure}{.25\textwidth}
  \centering
  \includegraphics[width=1.\linewidth]{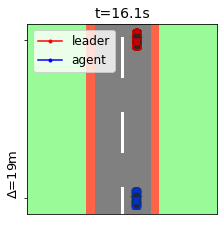}
  \caption{}
  \label{fig:driving_gcvar_frame}
\end{subfigure}
\begin{subfigure}{.25\textwidth}
  \centering
  \includegraphics[width=1.\linewidth]{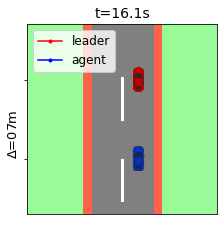}
  \caption{}
  \label{fig:driving_cesor_frame}
\end{subfigure}
\begin{subfigure}{.4\textwidth}
  \centering
  \includegraphics[width=1.\linewidth]{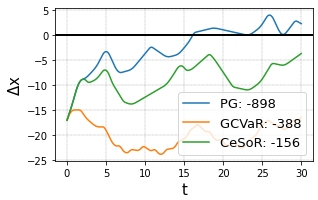}
  \caption{}
  \label{fig:driving_example_dx}
\end{subfigure}
\caption{\small (a-c) A sample frame in a test episode in the Driving Game. All the agents deal with the same situation (the same sequence of leader actions, which happened to include a sequence of decelerations).
While PG collides with the leader, CeSoR keeps a safe margin -- without losing as much distance as GCVaR.
Note that Figure~\ref{fig:driving_nice} effectively displays these 3 frames together.
(d) The agent-leader distance evolution in the whole episode, and the final episode score of each agent.}
\label{fig:driving_frames}
\end{figure}

\begin{figure}[!ht]
\centering
  \includegraphics[width=.9\linewidth]{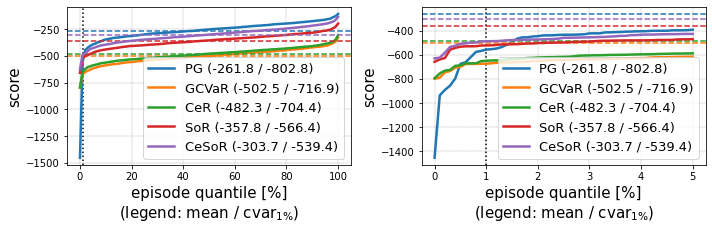}
\caption{\small Additional ablation tests for the Driving Game: the full returns distributions (left) and zoom-in to their tails (right). Note that we reran the experiment for the ablation test, resulting in slightly different returns than Figure~\ref{fig:awesome}. Both CeR and SoR lose to CeSoR in terms of CVaR and mean, indicating the necessity of both soft risk and CE-sampler in CeSoR.}
\label{fig:driving_ablation}
\end{figure}

\FloatBarrier


\section{The Computational Resource Allocation Problem: Extended Discussion}
\label{sec:servers_detailed}

\subsection{Implementation Details}
\label{sec:servers_implementation}

In this section we specify the implementation details of the Resource Allocation Problem presented in Section~\ref{sec:servers}.
The full code is available in the \code{Examples/ServersAllocation/ServersSim.py}{gym environment} and the corresponding \code{Examples/ServersAllocation/ServersExample.ipynb}{jupyter notebook}.

The benchmark simulates one-hour episodes, where user-requests arrive randomly and the agent is responsible to allocate sufficiently many servers to handle them.
Once a request is attended, its service time is distributed exponentially with an average of 1 second.
Every second $t$, the number of arrivals is distributed $\sim Exp(\lambda_t)$, where the arrival rate $\lambda_t$ is itself an exponential moving average (EMA) of the (unknown) users interest $r_t$, with a typical decay of 5 minutes (i.e., $\lambda_t=\frac{299}{5\cdot60}\lambda_{t-1}+\frac{1}{5\cdot60}r_t$). $r_t=3$ is usually constant, but an unpredictable event causes a peak load every second with probability $\phi_0=\frac{1}{3\cdot24\cdot3600}$, i.e., every 3 days (or 72 episodes) on average.
In case of a peak load we set the momentary user interest to $r_t=3\cdot 300$, which means that the arrival rate doubles immediately to $\lambda_t = \frac{299}{300}\lambda_{t-1} + \frac{1}{300}r_t = \frac{299}{300}3 + \frac{1}{300}3\cdot300 \approx 6$, and then starts decreasing exponentially back to 3, with a typical decay of 5 minutes.


Every minute, the agent observes the number of active servers $3\le n^s\le10$ (initialized every episode to $n^s=4$) and the number of pending user-requests in the system, and may choose to add or remove one server (or to keep the number of servers as before).
Uploading a new server takes a 2-minute delay before the server is ready to handle requests. Removing a busy server takes effect once the server ends its current task.
Note that the servers form an ordered list, and only the last server in the list can be directly removed.
This constraint has little significance, since (1) the queue of pending requests is a global FIFO queue (i.e., the assignment only happens when a server becomes available -- there is no separate queue per server); (2) the requests serving time is exponentially distributed, i.e., the remaining time of the current task is independent of the task history and thus is identical for all the busy servers at any point of time.

Denoting by $tts_i$ the Time-To-Service (TTS) latency of a request, the agent return is
$$ R=-\text{user cost} -\text{servers cost}=-\sum_{i\in\text{requests}}tts_i-2\sum_{t=1}^{3600} n^s_t .$$
Once a request is assigned to a server, its serving time$\,\sim Exp(1)$ is independent of the agent decisions. Thus, to simplify computations and to reduce the noise, we measure the TTS of a request only as the waiting time between arrival and beginning of serving.

We set a target risk level of $\alpha=0.01$, and train each agent for $n=100$ steps.
During the training, we gradually increase the episodes length $L$ from 15 to 60 seconds.
The CEM controls the peak events frequency $\phi$, or equivalently, the number of peaks per episode (which is distributed $\sim Binom(\phi, L)$). The update function of $\phi$ is simply the (weighted) average number of peaks per selected episode, divided by the episode length.
$\nu=50\%$ of the episodes per batch are drawn from the original distribution $D_{\phi_0}$.

Note that at times of no peak-loads, the arrival rate is $\lambda=3$ and the service rate equals the number of servers $n^s$ (since the service takes 1 second on average). Thus, in terms of queueing theory, any number of servers $n^s\ge4$ guarantees that the expected number of requests in the system is $E\left[n^r\right]=3/(n^s-3)\le 3$.
In particular, this means that the policy learned by PG (see Section~\ref{sec:servers}) chooses the minimal number of servers $n^s=4$ that can handle no-peak demand, and adds resources only when required.

The agent policy receives a 9-dimensional vector as an input. The first 8 elements correspond to a one-hot encoding of the current number of paid servers $3\le n^s \le 10$ (including new servers that are not finished uploading yet). The last element corresponds to the current number of pending user requests in the queue, divided by $10r=30$ (the average number of arriving requests in 10 seconds of no peak-load).



\subsection{Detailed Results}
\label{sec:servers_detailed_results}

Figure~\ref{fig:servers_test_scores} summarizes the test scores of the agents, where CeSoR presents a reduction of 44\% and 17\% in the CVaR cost in comparison to PG and GCVaR, respectively.
In addition, its average cost is only 7\% higher than PG, and 33\% lower than GCVaR.
That is, CeSoR significantly improves the CVaR return without as a large compromise to the mean as in GCVaR.
CeSoR also outperforms GCVaR in episodes both with and without peak events, as shown in Figure~\ref{fig:servers_scores_by_events} below.
As demonstrated in Figure~\ref{fig:servers_nice} and summarized in Figure~\ref{fig:servers_policy}, PG and CeSoR learned to allocate a default of 4 and 5 servers, respectively, and to react to peak loads as needed; whereas GCVaR simply allocates 8 servers at all times.

Note that the CE task -- sampling the bottom $\alpha=1\%$ -- is particularly challenging in this problem, due to the combination of very rare peak events and limited expressiveness of the Binomial distributions family. In particular, this family cannot guarantee the existence of a peak in a simulated episode without simulating \textit{multiple} peaks per episode (i.e., $P_{\phi^*}^{\pi_\theta} \ne P_{\phi_0,\alpha}^{\pi_\theta}$ in terms of Section~\ref{sec:variance_reduction}).
Yet, CeSoR is demonstrated robust to the poor parameterization selection of $D_\phi$, as it presents a reasonable sampling (see Appendix~\ref{sec:ce_sample_dist}) and improves the returns CVaR.

Figures~\ref{fig:servers_test_results}-\ref{fig:servers_examples} present a detailed analysis of the results. 

\begin{figure}[!ht]
\centering
\begin{subfigure}{.4\textwidth}
  \centering
  \includegraphics[width=1.\linewidth]{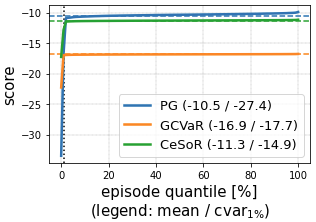}
  \caption{}
  \label{fig:servers_test_scores_full}
\end{subfigure}
\begin{subfigure}{.45\textwidth}
  \centering
  \includegraphics[width=1.\linewidth]{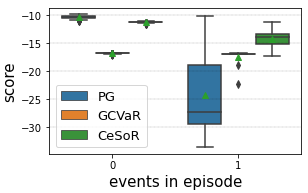}
  \caption{}
  \label{fig:servers_scores_by_events}
\end{subfigure}
\caption{\small (a) The full distribution of the trained agent returns over the test episodes in the Servers Allocation Problem. Note that Figure~\ref{fig:servers_test_scores} displays the left tail of the same distribution.
(b) A box-plot of the returns distribution for test episodes -- separately for episodes with and without a peak-overloading event. CeSoR achieves the best scores in episodes with peak events.}
\label{fig:servers_test_results}
\end{figure}

\begin{figure}[!ht]
\centering
\begin{subfigure}{.4\textwidth}
  \centering
  \includegraphics[width=1.\linewidth]{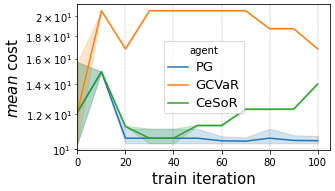}
  \caption{$Mean$}
  \label{fig:servers_valid_mean}
\end{subfigure}
\begin{subfigure}{.4\textwidth}
  \centering
  \includegraphics[width=1.\linewidth]{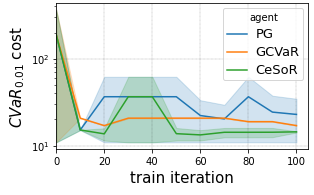}
  \caption{$CVaR_{1\%}$}
  \label{fig:servers_valid_cvar}
\end{subfigure}
\caption{\small Mean and CVaR scores over the validation episodes throughout the Servers Allocation Problem training. The shading corresponds to 95\% confidence-intervals, based on bootstrapping over the episode-samples.}
\label{fig:servers_valid}
\end{figure}

\begin{figure}[!ht]
\centering
  \includegraphics[width=.35\linewidth]{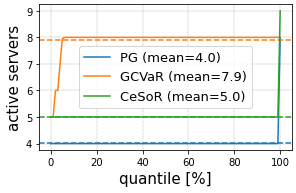}
\caption{\small The distribution of the number of servers allocated by each agent, over all the time-steps in all the test episodes. GCVaR allocates 8 servers in advance, whereas PG and CeSoR typically allocate 4 and 5 servers, respectively, and add servers as needed in case of overloading.}
\label{fig:servers_policy}
\end{figure}

\begin{figure}[!ht]
\centering
\begin{subfigure}{.9\textwidth}
  \centering
  \includegraphics[width=1.\linewidth]{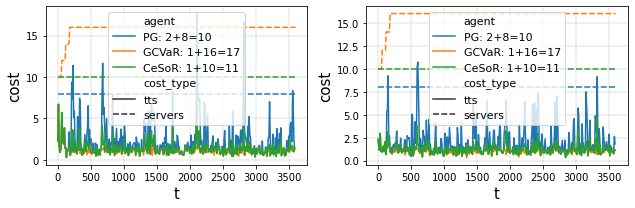}
  \caption{\small Two episodes with no peak events: all agents ave near-zero TTS-cost, and servers cost corresponding to their policy (which is itself shown in Figure~\ref{fig:servers_policy}).}
  \label{fig:servers_examples_calm}
\end{subfigure}
\begin{subfigure}{.9\textwidth}
  \centering
  \includegraphics[width=1.\linewidth]{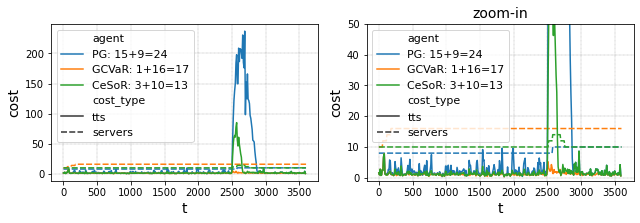}
  \caption{\small An episode with a peak event (right: zoom in around the event). This figure presents the same episode displayed in Figure~\ref{fig:servers_nice}, but normalizes the TTS and the servers allocation to the same units of cost, as defined by the benchmark. Notice that both PG and CeSoR react to the event with allocation of additional servers.}
  \label{fig:servers_example_overloading}
\end{subfigure}
\caption{\small A sample of test episodes in the Servers Allocation Problem. The legends specify the TTS-cost, the servers-cost and the total cost.}
\label{fig:servers_examples}
\end{figure}

\begin{figure}[!ht]
\centering
  \includegraphics[width=.4\linewidth]{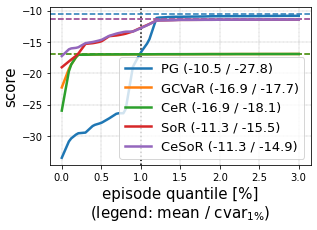}
\caption{\small Additional ablation tests for the Servers Allocation Problem. Note that we reran the experiment for the ablation test, resulting in slightly different returns than Figure~\ref{fig:awesome}. Both CeR and SoR lose to CeSoR in terms of CVaR and mean, indicating the necessity of both soft risk and CE-sampler in CeSoR.}
\label{fig:servers_ablation}
\end{figure}

\FloatBarrier


\section{Distributional Reinforcement Learning for CVaR Optimization}
\label{sec:dRL}

Many RL algorithms aim to learn the value $Q(s,a)$ of a state-action pair, representing the expected return from choosing action $a$ at state $s$. Then, given a state and a finite set of action candidates, the agent can choose the action with the highest value.
In Distributional Reinforcement Learning (DRL), not only the expected return is learned, but rather the whole return distribution -- conditioned on $s$, $a$ and the current policy.
While standard DRL algorithms~\citep{dRL,QRDQN} still optimize the expected return and thus are risk-neutral, the learning of the whole return distribution encourages risk-averse variants as well~\citep{distributional_rl_risk}.

A naive risk-averse DRL agent may simply use the learned return distribution to choose the action with the highest risk measure (e.g., CVaR) over the returns.
However, notice that the return distribution is conditioned on the policy.
Hence, similarly to other RL methods, the learned values become incorrect once we change the policy: the CVaR of the current action does not take into account the change in the next action.
Thus, this naive approach would not truly optimize the CVaR.

Instead, a risk-averse DRL agent can train using a risk-averse actor, such that the learned distribution is consistent with the risk-averse policy.
This approach is valid and was indeed used by \citet{distributional_rl_risk}.
However, it suffers from similar limitations as CVaR-PG.
Regarding sample-efficiency, CVaR-DRL considers only the bottom quantiles of the distribution, whose corresponding loss function assigns very low weights to all the returns except for the lowest ones, reducing the effective sample size. In particular, since there is no separation between low returns and high-risk environment conditions, still only a small portion of the data corresponds to high-risk, and it remains challenging to learn how to act under such conditions.
Regarding blindness to success, CVaR-DRL is still prone to miss beneficial strategies: it still directs the actor policy according to the lowest returns rather than the hardest conditions, and learns the distribution with respect to that policy.

We implemented the methods mentioned above for the Guarded Maze benchmark, on top of the QR-DQN~\citep{QRDQN} implementation of Stable-Baselines~\citep{stable-baselines3}.
As shown in Table~\ref{tab:drl}, none of the DRL variants improved the CVaR return even in comparison to the baseline CVaR-PG (GCVaR):
the standard risk-neutral QR-DQN obtained similar returns to the risk-neutral PG;
the naive DRL approach resulted in a noisy and seemingly-meaningless policy, obtaining worse returns than GCVaR;
and the valid CVaR-DRL obtained identical returns to GCVaR.

These results support the discussion above, indicating that blindness to success and sample-inefficiency are general limitations in risk-averse RL, and in particular apply to DRL in addition to PG.
We hope that our work will pave the way for other efficient risk-averse RL methods, beyond the scope of PG algorithms.

\begin{table}[h]
\caption{\small A comparison of CeSoR test returns to both PG and Distributional RL methods, over the Guarded Maze benchmark. The first two methods are risk-neutral.}
\label{tab:drl}
\vspace{2pt}
\centering
\begin{tabular}{|c|cc|}
\hline
\textbf{Algorithm}               & \textbf{Mean} & \textbf{CVaR\textsubscript{0.05}} \\ \hline
PG                               & \textbf{4}    & -63           \\
QR-DQN                           & \textbf{3}    & -73           \\ \hline
GCVaR                            & -32           & -32           \\
CVaR-QR-DQN (only inference)     & -32           & -39           \\
CVaR-QR-DQN (training+inference) & -32           & -32           \\
CeSoR                            & 2             & \textbf{-7}   \\ \hline
\end{tabular}
\end{table}


\end{document}